%% file: ms.tex
\pdfoutput=1

\documentclass[twoside]{article}

\usepackage[accepted]{aistats2025}
\usepackage[round]{natbib}

\usepackage{bm}
\usepackage{xcolor}
\usepackage{amsmath}
\usepackage{amssymb}
\usepackage{amsthm}
\usepackage{amsfonts}       
\usepackage{algorithm}
\usepackage{algorithmic}
\usepackage{subcaption}
\usepackage{graphicx}
\graphicspath{ {./img/} }

\usepackage{Sty/mcr}

\newtheorem{lemma}{Lemma}
\newtheorem{remark}{Remark}
\newtheorem{example}{Example}

\DeclareMathOperator*{\argmin}{arg\,min}

\begin{document}

%
\runningtitle{Statistical Inference for Feature Selection after Optimal Transport-based Domain Adaptation}

%
\runningauthor{Nguyen Thang Loi, Duong Tan Loc, Vo Nguyen Le Duy}

\twocolumn[

\aistatstitle{Statistical Inference for Feature Selection after \\Optimal Transport-based Domain Adaptation}

\aistatsauthor{Nguyen Thang Loi$^{1, 2}$, Duong Tan Loc$^{1, 2}$, Vo Nguyen Le Duy$^{1, 2, 3, \ast}$}
\vspace{5pt}
\aistatsaddress{ 
$^1$University of Information Technology, Ho Chi Minh City, Vietnam \\ 
$^2$Vietnam National University, Ho Chi Minh City, Vietnam \\
$^3$RIKEN
}]

\begin{abstract}
\input{abst}
\end{abstract}

\input{sec1}

\input{sec2}

\input{sec3}

\input{sec4}
\input{sec5}

\bibliographystyle{abbrvnat}
\bibliography{ref}

\newpage
\onecolumn

\input{appendix}

\end{document}

%% file: abst.tex
Feature Selection (FS) under domain adaptation (DA) is a critical task in machine learning, especially when dealing with limited target data.
However, existing methods lack the capability to guarantee the reliability of FS under DA. 
In this paper, we introduce a novel statistical method to statistically test FS reliability under DA, named SFS-DA (statistical FS-DA). 
The key strength of SFS-DA lies in its ability to control the false positive rate (FPR) below a pre-specified level $\alpha$ (e.g., 0.05) while maximizing the true positive rate.
Compared to the literature on statistical FS, SFS-DA presents a unique challenge in addressing the effect of DA to ensure the validity of the inference on FS results.
We overcome this challenge by leveraging the Selective Inference (SI) framework. 
Specifically, by carefully examining the FS process under DA whose operations can be characterized by linear and quadratic inequalities, we prove that achieving FPR control in SFS-DA is indeed possible.
Furthermore, we enhance the true detection rate by introducing a more strategic approach.
Experiments conducted on both synthetic and real-world datasets robustly support our theoretical results, showcasing the SFS-DA’s superior performance.

%% file: sec1.tex

\section{Introduction} \label{sec:intro}

Feature selection (FS) is an important task in machine learning (ML), aimed at identifying the most relevant features from a dataset while discarding redundant or irrelevant ones. 
By reducing the dimensionality of the data, FS enhances model interpretability, reduces overfitting, and improves computational efficiency. 
It plays a critical role in high-dimensional data settings, where the number of features often exceeds the number of observations. 
Common FS techniques such as Lasso \citep{tibshirani1996regression} and stepwise feature selection plays a critical role in several applications and has been widely applied in many areas such as economics and finance \citep{tian2015variable, coad2020catching}, bioinformatics \citep{wu2009genome, ma2008penalized}, and chemoinformatics \citep{lo2018machine}.

In many applications, limited data availability can impair the performance of FS models.
Domain adaptation (DA) offers a solution by allowing models to transfer data points from a source domain with abundant labeled data to a target domain with limited labeled data. 
This approach capitalizes the similarities between the two domains and utilizes techniques such as optimal transport to align their distributions, thereby improving the efficacy of FS in the target domain where limited data hinder effectiveness and boosting model performance in practical applications.

When conducting FS under DA, there is a critical risk of mistakenly selecting irrelevant features as relevant.
These erroneous FS results are commonly referred to as \emph{false positives}.
In high-stakes applications, such as medical diagnostics, these false positives can cause serious consequences.
For example, selecting irrelevant features may result in incorrectly identifying a patient as high-risk for breast cancer due to unrelated genetic markers. 
Such a misidentification could lead to unnecessary procedures, such as biopsies or preventive surgeries, causing emotional distress, financial burdens, and potential physical harm to the patient—all without any actual need for treatment.
 This highlights the importance of developing a statistical method that properly controls the false positive rate (FPR).

\begin{figure*}[hbt!]
\centering
\includegraphics[width=.85\textwidth]{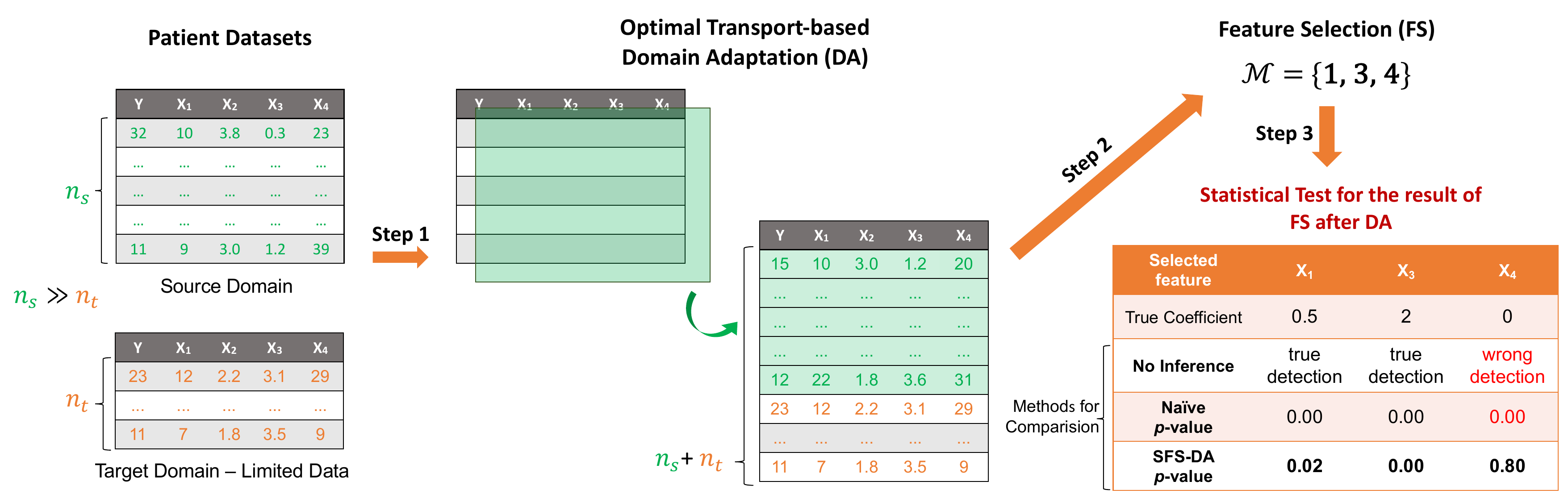}
\caption{Illustration of the proposed SFS-DA method. Performing FS-DA without statistical inference leads to the selection of irrelevant feature ($X_4$).
The naive $p$-value is small even for a falsely detected feature. 
%
With the proposed SFS-DA method, we successfully identify both false positives (FPs) and true positives (TPs), yielding large $p$-values for FPs and small $p$-values for TPs.}
    \label{fig:SFS_DA}
    \vspace{-10pt}
\end{figure*}

We also emphasize the critical need to manage the false negative rate (FNR).
In statistical literature, a common strategy involves initially controlling the false positive rate (FPR) at a pre-determined level, such as $\alpha = 0.05$, while concurrently aiming to minimize the FNR, which is equivalent to maximizing the true positive rate (TPR = 1 - FNR) through empirical evidence. Following this established methodology, this paper adopts a similar approach.
We propose a method to theoretically control the probability of misclassifying an irrelevant feature as relevant while simultaneously minimizing the probability of misclassifying a relevant feature as irrelevant.

To the best of our knowledge, \emph{none} of the existing methods can control the FPR of FS under DA. 
Several methods have been proposed in the literature for FPR control in FS techniques, including Lasso \citep{berk2013valid, lee2016exact, duy2021more} and stepwise feature selection \citep{tibshirani2016exact, sugiyama2021more}. 
However, these methods assume that the data originates from the same distribution, which becomes \emph{invalid} in scenarios where a distribution shift occurs and DA must be employed.
\citet{le2024cad} is the first work capable of conducting the inference under DA. However, their method primarily focuses on the unsupervised anomaly detection problem, which completely differs from the setup of FS under DA we consider in this paper. Consequently, their approach cannot be applied to our setting.

Conducting valid inference for controlling the FPR in FS under DA is challenging because the features selected are dependent on the application of FS-DA to the data. This violates the assumption of traditional inference methods, which require that the selected features be fixed in advance. 
To overcome the challenge, our idea is to leverage the concept of \emph{Selective Inference (SI)} \citep{lee2016exact}. 
However, directly applying SI in our setting is non-trivial, as SI is inherently problem- and model-specific. 
This necessitates the development of a new method tailored to the specific setting and the structure of the ML model.
Consequently, we need to thoroughly examine the algorithm's selection strategy in the context of FS under DA.

In this paper, we focus on Optimal Transport (OT)-based DA \citep{flamary2016optimal}, which has gained popularity in the OT community, as well as the Lasso, a well-established method for FS.
Additionally, we extend our method to the elastic net \citep{zou2005regularization}.
The detailed discussions on future extensions to other types of FS and DA are provided in \S \ref{sec:discussion}.

\textbf{Contributions.} Our contributions are as follows:

$\bullet$ We introduce and mathematically formulate the problem setup of testing FS results in the context of DA within the hypothesis testing framework.
We presents a unique challenge in addressing the impact of FS under DA to ensure the validity of FPR control.

$\bullet$ We propose a novel statistical method, named \emph{SFS-DA} (statistical FS-DA), to conduct the introduced hypothesis test.
To our knowledge, this is the first method capable of properly the FPR in FS under DA by providing valid $p$-values for the selected features.
Additionally, we introduce a more strategic approach to maximize the TPR, i.e., reducing the FNR.

$\bullet$ We conduct extensive experiments on both synthetic and real-world datasets to thoroughly validate our theoretical findings, demonstrating the superior performance of the proposed SFS-DA method.

\begin{table}[!t]
\renewcommand{\arraystretch}{1.1}
\centering
\caption{The key strength of the proposed SFS-DA lies in its ability to control the False Positive Rate (FPR).}
\vspace{-5pt}
\begin{tabular}{ |l|c|c|c| } 
  \hline
  & \textbf{No Inference} & \textbf{Naive} & \textbf{SFS-DA} \\
  \hline
  $N = 120$ & FPR = 1.0 & 0.15 & \textbf{0.05} \\
   \hline
  $N = 240$ & FPR = 1.0 & 0.12 & \textbf{0.04} \\
  \hline
\end{tabular}
\label{tbl:example_intro}
\vspace{-10pt}
\end{table}

\begin{example}
To demonstrate the importance of the proposed SFS-DA method, we present an example in Fig. \ref{fig:SFS_DA}.
Our objective is to perform FS to identify the relevant features that influence blood glucose level in the target domain, e.g., a hospital with a limited number of patients.
We employ the OT-based DA approach to transfer the data from the source domain, where we have a substantial patient dataset, to the target domain.
Subsequently, we apply a FS algorithm, i.e., the Lasso.
The FS after DA erroneously identified an irrelevant feature as relevant.
To resolve this issue, we introduced an additional inference step using the SFS-DA $p$-values, enabling us to identify both true positive and false positive detections.
Additionally, we repeated the experiments $N$ times, with the FPR results presented in Tab. \ref{tbl:example_intro}.
Using the proposed method, we successfully controlled the FPR at $\alpha = 0.05$, which other competing methods were unable to achieve.
\end{example}

\textbf{Related works.} Traditional statistical inference in feature selection often faces challenges regarding the validity of $p$-values. 
A common issue arises from the reliance on naive $p$-values, which are computed under the assumption that the selected features are fixed in advance. 
However, when features are selected using the FS-DA method, this assumption is violated, which makes the naive $p$-values invalid.
%
Data splitting (DS) offers a solution by dividing the data into two parts: one for selection and the other for inference.
This ensures that the feature selection phase is independent of the testing phase, making the $p$-values computed via DS valid. 
However, DS reduces the amount of data available for both phases, potentially weakening the statistical power.
Additionally, it is not always possible to split the data, e.g., when the data is correlated.

In recent years, SI has been actively studied for conducting inference on the features of linear models that are selected by FS methods.
SI was first introduced 
for the Lasso \citep{lee2016exact}.
The basic idea of SI is to conduct the inference conditional on the FS process.
This approach mitigates the bias of the FS step, allowing for the computation of valid $p$-values. 
The seminal paper laid the foundation for subsequent research on SI for FS \citep{loftus2014significance, fithian2014optimal, tibshirani2016exact, yang2016selective, suzumura2017selective, hyun2018exact, sugiyama2021more, das2022fast, duy2021more}. 
However, these methods assume the data is drawn from the same distribution.
Therefore, they lose the validity in the context of DA, where distribution shifts occur, making them inappropriate for such scenarios.

A closely related work, and the main motivation for this study, is \cite{le2024cad}, where the authors propose a framework for computing valid $p$-values for anomalies detected by an anomaly detection method within an OT-based DA setting. 
However, their focus is on unsupervised learning and anomaly detection task, which completely differs from the problem setup of supervised FS under DA that we consider in this paper. 
As a result, their method is not directly applicable to our setting.

%% file: sec2.tex
\section{Problem Setup} \label{sec:problem_setup}

To formulate the problem, we consider a regression setup with two random response vectors defined by
\begin{align*}
	\bm Y^s = (Y_1^s, \dots, Y_{n_s}^s)^\top \sim \mathbb{N}(\bm \mu^s, \Sigma^s), \\
	\bm Y^t = (Y_1^t, \dots, Y_{n_t}^t)^\top \sim \mathbb{N}(\bm \mu^t, \Sigma^t),
\end{align*}
where $n_s$ and $n_t$ are the number of instances in the source and target domains, $\bm \mu^s$ and $\bm \mu^t$ are unknown signals, $\bm \veps^s$ and $\bm \veps^t$ are the Gaussian noise vectors with the covariance matrices $\Sigma^s$ and $\Sigma^t$ assumed to be known or estimable from independent data. 
We denote the feature matrices in the source and target domains, which are non-random, by $X^s \in \RR^{n_s \times p}$ and $X^t \in \RR^{n_t \times p}$, 
respectively, where $p$ is the number of features.
We assume that the number of instances in the target domain is limited, i.e., $n_t$ is much smaller than $n_s$.
The goal is to statistically test the Lasso results after DA.

\subsection{Optimal Transport (OT)-based DA}
We leverage the OT-based DA  proposed by \citet{flamary2016optimal} and apply it to our supervised setting. 
%
Let us define the source and target data as:
\begin{align} \label{eq:Ds_Dt}
	D^s &= 
	\begin{pmatrix}
		X^s ~ \bm Y^s
	\end{pmatrix} \text{ and }
	D^t = 
	\begin{pmatrix}
		X^t ~ \bm Y^t
	\end{pmatrix}, 
\end{align}
$D^s \in \RR^{n_s \times (p + 1)}$, $D^t \in \RR^{n_t \times (p + 1)}$.
%
%
Then, we define the the cost matrix as:
\begin{align*}
	C(D^s, D^t) 
	& = \Big[
	\big \| D_i^s - D_j^t \big \|^2_2 
	\Big]_{ij} \in \RR^{n_s \times n_t},
\end{align*}
for any $i \in [n_s] = \{1, 2, ..., n_s\}$ and $j \in [n_t]$.
We note that $D_i^s$ and $D_j^t$ are the $i^{\rm th}$ and $j^{\rm th}$ rows of $D^s$ and $D^t$, respectively, which are the $(p + 1)$-dimensional vectors.

\textbf{Optimal transport}. The OT problem for DA between the source and target domains is defined as:
\begin{align} \label{eq:ot_problem}
		&\hat{T} = \argmin 
		\limits_{T \in \RR^{n_s \times n_t}, ~T \geq 0} 
		~ \big \langle T, C(D^s, D^t) \big \rangle \\ 
		& \quad \quad \quad \quad ~ \text{s.t.} ~~~ T \bm{1}_{n_t} = \bm 1_{n_s}/{n_s}, ~ T^\top \bm{1}_{n_s} = \bm 1_{n_t}/{n_t}  \nonumber,
\end{align}
where $\langle\cdot,\cdot\rangle$ is the Frobenius inner product, $\bm{1}_n \in \RR^n$ is the vector whose elements are set to $1$.
Once the optimal transportation matrix $\hat{T}$ is obtained, the source instances are transported into the target domain.

\textbf{Transformed data after DA.}
The transformation $\tilde{D}^s$ of $D^s$ is defined as:
\begin{align} \label{eq:tilde_D_s}
	\tilde{D}^s 
		= n_s \hat{T} D^t \in \RR^{n_s \times (p + 1)}.
\end{align} 
More details are provided in Sec 3.3 of \cite{flamary2016optimal}. 
Let us decompose $\tilde{D}^s$ into 
$
	\tilde{D}^s = 
	\begin{pmatrix}
		\tilde{X}^s ~ \tilde{\bm Y}^s
	\end{pmatrix},
$
the matrix $\tilde{X}^s$ and vector $\tilde{\bm Y}^s$ can be defined as:
\begin{align} \label{eq:data_after_da}
	\tilde{X}^s = n_s \hat{T} X^t
	\quad 
	\text{and}
	\quad
	\tilde{\bm Y}^s = n_s \hat{T} \bm Y^t,
\end{align}
according to \eq{eq:tilde_D_s} and the definition of $D^t$ in \eq{eq:Ds_Dt}.
Here, $\tilde{X}^s$ and $\tilde{\bm Y}^s$ represent the transformations of $X^s$ and $\bm Y^s$ to the target domain, respectively.

\subsection{Feature Selection by Lasso after DA}

After transforming the data from the source domain to the target domain, we apply the Lasso to the combined dataset of the transformed and target data:
\begin{equation} \label{eq:lasso}
	\hat{{\bm \beta}} = \argmin 
	\limits_{{\bm \beta} \in \RR^p} \frac{1}{2} 
	\big \|
	\tilde{\bm Y} - \tilde{X} {\bm \beta}
	\big \|^2_2 + \lambda \|{\bm \beta}\|_1,
\end{equation}
where $\lambda \geq 0$ is a regularization parameter, 
\begin{align*} 
	\tilde{X} = 
	\begin{pmatrix}
		\tilde{X}^s \\ 
		X^ t
	\end{pmatrix}
	\in \RR^{(n_s + n_t) \times p}, 
	~
	\tilde{\bm Y} = 
	\begin{pmatrix}
		\tilde{\bm Y}^s \\ 
		\bm Y^ t
	\end{pmatrix}
	\in \RR^{n_s + n_t}
\end{align*}
Since the Lasso produces sparse solutions, the set of selected features is defined as:
\begin{align} \label{eq:selected_features}
	\cM = \big \{ j : \hat{\beta}_j \neq 0 \big \}.
\end{align}
While our primary focus in the main paper is on the Lasso, we also present an extension to the elastic net \citep{zou2005regularization}. Detailed information is provided in \S \ref{sec3:extension_elastic_net}. 
Future extensions to other types of FS methods are discussed in \S \ref{sec:discussion}.

\subsection{Statistical Inference and Decision Making}

Our goal is to assess if the selected features in \eq{eq:selected_features} are truly relevant or  just selected by chance. 
To conduct the inference on the $j^{\rm th}$ selected feature, we consider the statistical test on the following hypotheses:
\begin{equation*}
 {\rm H}_{0, j}:  \beta_j  = 0 \quad \text{vs.} \quad {\rm H}_{1, j}:  \beta_j \neq 0,
\end{equation*}
where $\beta_j = \left [ \big ({X^t_{\cM}}^\top X^t_{\cM} \big)^{-1} {X^t_{\cM}}^\top \bm \mu^t \right ]_j$ and $X^t_{\cM}$ is the sub-matrix of $X^t$ made up of columns in the set $\cM$.
To test these hypotheses, a natural choice of the test statistic is the least square estimate defined as:
\begin{align} \label{eq:test_statistic}
	\tau_j = \left [ \big ({X^t_{\cM}}^\top X^t_{\cM} \big)^{-1} {X^t_{\cM}}^\top \bm Y^t \right ]_j
	= \bm \eta_j^\top {\bm Y^s \choose \bm Y^t },
\end{align}
where $\bm \eta_j$ is the direction of the test statistic: 
\begin{align} \label{eq:eta_j}
\bm \eta_j = 
\begin{pmatrix}
	\bm 0^{s} \\ 
	X^t_{\cM} \big ({X^t_{\cM}}^\top X^t_{\cM} \big)^{-1} \bm{e}_j
\end{pmatrix},
\end{align}
in which $\bm 0^{s} \in \RR^{n_s}$ represents a vector where all entries are set to 0, 
$\bm e_j \in \RR^{|\cM|}$ is a vector in which the $j^{\rm th}$ entry is set to $1$, and $0$ otherwise.

\textbf{Compute $p$-value and decision making.} After obtaining the test statistic in \eq{eq:test_statistic}, we proceed to compute a $p$-value.
Given a significance level $\alpha \in [0, 1]$, e.g., 0.05, we reject the null hypothesis ${\rm H}_{0, j}$ and assert that the $j^{\rm th}$ feature is relevant if the $p$-value $ \leq \alpha$.
Conversely, if the $p$-value $ > \alpha$, there is not enough evidence to conclude that the $j^{\rm th}$ feature is relevant.

\paragraph{Challenge of computing a valid $p$-value.}
The conventional (naive) $p$-value is defined as: 
\begin{align*}
	p^{\rm naive}_j = 
	\mathbb{P}_{\rm H_{0, j}} 
	\Bigg ( 
		\left | \bm \eta_j^\top {\bm Y^s \choose \bm Y^t } \right |
		\geq 
		\left | \bm \eta_j^\top {\bm Y^s_{\rm obs} \choose \bm Y^t_{\rm obs} } \right |
	\Bigg ), 
\end{align*}
where $\bm Y^s_{\rm obs}$ and $\bm Y^t_{\rm obs}$ are the observations of the random vectors $\bm Y^s$ and $\bm Y^t$, respectively.
If the vector $\bm \eta_j$ is independent of the FS and DA algorithms, the naive $p$-value is valid in the sense that 
\begin{align} \label{eq:valid_p_value}
	\mathbb{P} \Big (
	\underbrace{p_j^{\rm naive} \leq \alpha \mid {\rm H}_{0, j} \text{ is true }}_{\text{a false positive}}
	\Big) = \alpha, ~~ \forall \alpha \in [0, 1],
\end{align} 
i.e., the probability of obtaining a false positive is controlled under a certain level of guarantee.
However, in our setting, the vector $\bm \eta_j$ is influenced by the FS and DA, i.e., it is defined based on the set of selected features after performing FS under DA.
As a result, the property of a valid $p$-value in \eq{eq:valid_p_value} is no longer satisfied.
Consequently, the naive $p$-value is \emph{invalid} because it does not account for  the effect of FS and DA.

%% file: sec3.tex
\section{Proposed SFS-DA Method} \label{sec:method}

\begin{figure*}[!t]
    \centering
    \includegraphics[width=.8\textwidth]{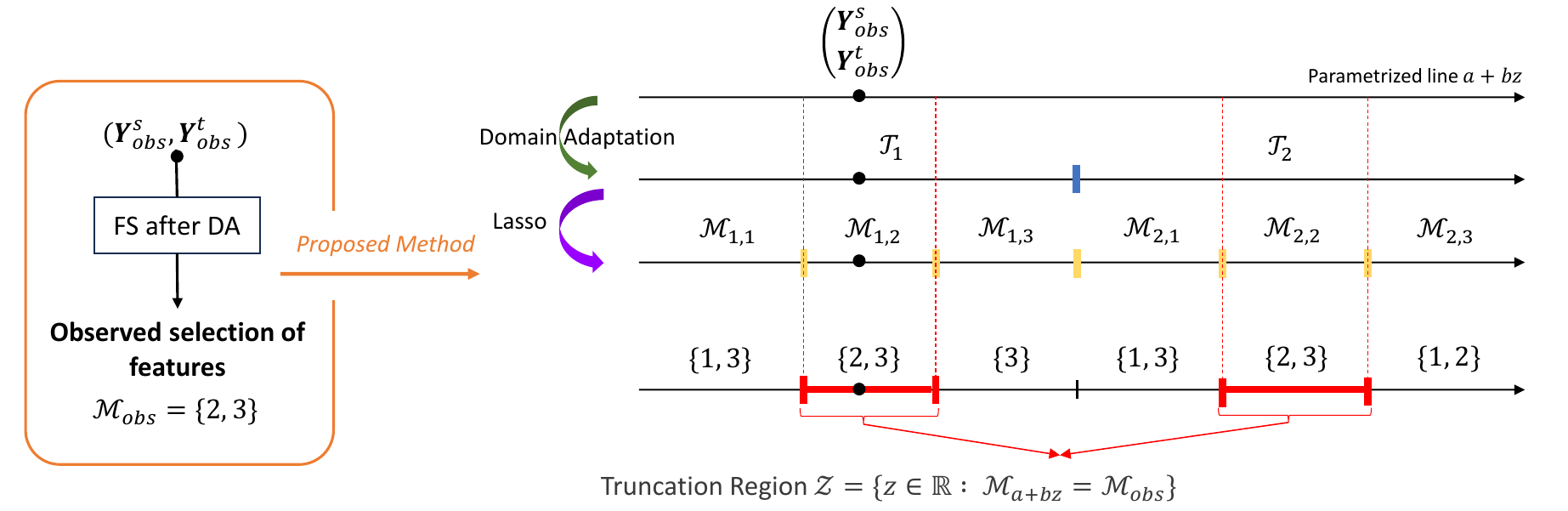}
    \caption{After performing DA, we apply FS to identify the relevant features. 
    Next, we parametrize the data using a scalar parameter $z$ in the dimension of the test statistic to define the truncation region $\cZ$, whose the data have the \emph{same} FS results as the observed data. Finally, we conduct the inference by conditioning on $\cZ$. To enhance the efficiency, we utilize a divide-and-conquer strategy to effectively identify the region $\cZ$.}
    \label{fig:line_search_approach}
    \vspace{-10pt}
\end{figure*}

In this section, we present the details of the proposed SFS-DA method for computing the valid $p$-value.

\subsection{The valid $p$-value in SFS-DA}

To compute the valid $p$-value, we first need to determine the distribution of the test statistic defined in \eq{eq:test_statistic}.
We achieve this by leveraging the concept of SI, specifically by examining the distribution of the test statistic \emph{conditional} on the FS results after DA:
\begin{align} \label{eq:conditional_distribution}
	\mathbb{P} \Bigg ( 
	\bm \eta_j^\top {\bm Y^s \choose \bm Y^t }
	~
	\Big |
	~ 
	\cM_{\bm Y^s, \bm Y^t}
	=
	\cM_{\rm obs}
	\Bigg ),
\end{align}
where $\cM_{\bm Y^s, \bm Y^t}$ is the set of selected features of Lasso FS after DA \emph{for any random} vectors $\bm Y^s$ and $\bm Y^t$, and $\cM_{\rm obs} = \cM_{\bm Y^s_{\rm obs}, \bm Y^t_{\rm obs}}$ is the observed selected features.

Based on the distribution in \eq{eq:conditional_distribution}, we introduce the selective $p$-value which is defined as:
\begin{align} \label{eq:selective_p}
	p^{\rm sel}_j = 
	\mathbb{P}_{\rm H_{0, j}} 
	\Bigg ( 
		\left | \bm \eta_j^\top {\bm Y^s \choose \bm Y^t } \right |
		\geq 
		\left | \bm \eta_j^\top {\bm Y^s_{\rm obs} \choose \bm Y^t_{\rm obs} } \right |
		~
		\Bigg | 
		~
		\cE
	\Bigg ), 
\end{align}
where $\cE$ is the conditioning event defined as
\begin{align} \label{eq:conditioning_event}
\cE = \Big \{ 
	\cM_{\bm Y^s, \bm Y^t}
	=
	\cM_{\rm obs}, ~
	\cQ_{\bm Y^s, \bm Y^t}
	=
	\cQ_{\rm obs}
\Big \}. 
\end{align}
The $\cQ_{\bm Y^s, \bm Y^t}$ is the \emph{nuisance component} defined as 
\begin{align} \label{eq:q_and_b}
	\cQ_{\bm Y^s, \bm Y^t} = 
	\Big ( 
	I_{n_s + n_t} - 
	\bm b
	\bm \eta_j^\top \Big ) 
	{\bm Y^s \choose \bm Y^t},
\end{align}
where 
$
	\bm b = \frac{\Sigma \bm \eta_j}
	{\bm \eta_j^\top \Sigma \bm \eta_j}
$
and 
$
\Sigma = 
\begin{pmatrix}
	\Sigma^s & 0 \\ 
	0 & \Sigma^t
\end{pmatrix}.
$

\begin{remark}
The nuisance component $\cQ_{\bm Y^s, \bm Y^t}$ corresponds to
the component $\bm z$ in the seminal paper of \cite{lee2016exact} (see Sec. 5, Eq. (5.2)).
The additional conditioning on $\cQ_{\bm Y^s, \bm Y^t}$ is required for technical reasons, specifically to facilitate tractable inference.
This is the standard approach in SI literature and it is used in almost all the SI-related works that we cite.

\begin{lemma} \label{lemma:valid_selective_p}
The selective $p$-value proposed in \eq{eq:selective_p} satisfies the property of a valid $p$-value:
\begin{align*}
	\mathbb{P}_{\rm H_{0, j}}  \Big (
	p_j^{\rm sel} \leq \alpha
	\Big) = \alpha, ~~ \forall \alpha \in [0, 1].
\end{align*} 
\end{lemma}

\begin{proof}
The proof is deferred to Appendix \ref{app:proof_valid_p}.
\end{proof}
\end{remark}

Lemma \ref{lemma:valid_selective_p} indicates that, by using the proposed selective $p$-value, the FPR is theoretically controlled for any level $\alpha \in [0, 1]$.
Once  $\cE$ in \eq{eq:conditioning_event} is identified, the selective $p$-value can be computed.
We will present the characterization of $\cE$ in the next section.

\subsection{Characterization of Conditioning Event $\cE$ }

Let us define the set of ${\bm Y^s \choose \bm Y^t } \in \RR^{n_s + n_t}$ that satisfies the conditions in $\cE$ defined in \eq{eq:conditioning_event} as:
\begin{align} \label{eq:conditional_data_space}
	{\hspace{-2mm}}\cD = \left \{ 
	{\bm Y^s \choose \bm Y^t }
	\in \RR^{n_s + n_t}
	 ~\Big | ~
	\begin{array}{l}
	\cM_{\bm Y^s, \bm Y^t}
	=
	\cM_{\rm obs}, \\
	\cQ_{\bm Y^s, \bm Y^t}
	=
	\cQ_{\rm obs}
	\end{array}
	\right \}. 
\end{align}
In the following lemma, we show that the conditional data space $\cD$ is, in fact,  restricted to a \emph{line} in $\RR^n$.
\begin{lemma} \label{lemma:data_line}
Let us define $\bm a = \cQ_{\rm obs}$, the set $\cD$ in \eq{eq:conditional_data_space} can be rewritten using a scalar parameter $z \in \RR$ as:
\begin{align} \label{eq:conditional_data_space_line}
	\cD = \left \{ {\bm Y^s \choose \bm Y^t } = \bm Y(z) = \bm a + \bm b z \mid z \in \cZ \right \},
\end{align}
where $\bm b$ is defined in \eq{eq:q_and_b} and $\cZ$ is defined as:
\begin{align} \label{eq:cZ}
	\cZ = \Big \{ 
	z \in \RR 
	\mid 
	\cM_{\bm a + \bm b z} = \cM_{\rm obs}
	\Big \}.
\end{align}
Here, 
$
\cM_{\bm a + \bm b z} = \cM_{{\bm Y^s \choose \bm Y^t }}
$
is equivalent to $\cM_{\bm Y^s, \bm Y^t}$.
\end{lemma}
%
%
The proof is deferred to Appendix \ref{app:proof_line}. The fact that the conditional space can be restricted to a line was implicitly exploited in \citep{lee2016exact} and discussed in Sec. 6 of \citep{liu2018more}.
Lemma \ref{lemma:data_line} shows that it is not necessary to consider the $n$-dimensional data space.
Instead, we only need to focus on the \emph{one-dimensional projected} data space $\cZ$ in \eq{eq:cZ}.

\textbf{Reformulation of the selective \textit{p}-value computation with $\cZ$.}
Let us denote a random variable $Z \in \RR$ and its observation $Z_{\rm obs} \in \RR$ as follows:
\begin{align*}
	Z = \bm \eta_j^\top {\bm Y^s \choose \bm Y^t } \in \RR 
	~~ \text{and} ~~ 
	Z_{\rm obs} = \bm \eta_j^\top {\bm Y^s_{\rm obs} \choose \bm Y^t_{\rm obs} } \in \RR.
\end{align*}
%
%
The selective $p$-value in (\ref{eq:selective_p}) can be rewritten as 
\begin{align} \label{eq:selective_p_reformulated}
	p^{\rm sel}_j = \mathbb{P}_{\rm H_{0, j}} \Big ( |Z| \geq |Z_{\rm obs}| \mid  Z \in \cZ \Big ).
\end{align}
Once the truncation region $\cZ$ is identified, computation of the selective $p$-value in (\ref{eq:selective_p_reformulated}) is straightforward.
Therefore, the remaining task is to identify $\cZ$.

\subsection{Identification of Truncation Region $\cZ$}
\label{subsec:identification_cZ}
To identify $\cZ$, the naive approach is to apply Lasso FS under DA on 
$
{\bm Y^s \choose \bm Y^t } = \bm a + \bm b z
$
for \emph{infinitely many} values of $z \in \RR$ to obtain the set of features $\cM_{\bm a + \bm b z}$ and check if it is the same as the observed $\cM_{\rm obs}$ or not, which is \emph{computationally intractable}.
To resolve the difficulty, we introduce an efficient approach (illustrated in Fig. \ref{fig:line_search_approach}), inspired by \citet{duy2021more, le2024cad}, to identify $\cZ$ in finite operations as follows:

%
%

$\bullet$ We divide the problem into multiple sub-problems, conditioning not only on the set of selected features but also on the DA transportation and the signs of the coefficients of the selected feature.

$\bullet$ We show that the sub-problem is efficiently solvable.

$\bullet$ We combine multiple sub-problems to obtain $\cZ$.

\textbf{Divide-and-conquer strategy.}
Let us denote by $U$ a total number of possible transportations for DA along the parametrized line.
We define $V_u$ as a number of all possible sets of features can be obtained by Lasso FS after the $\cT_u$ transportation, $u \in [U]$.
The entire one-dimensional space $\RR$ can be decomposed as:
\begin{align*}
	\RR
	& = 
	\bigcup \limits_{u \in [U]}
	\bigcup \limits_{v \in [V_u]}
	\underbrace{
	\left \{
	z  \in \RR
	~ \bigg | 
	\begin{array}{l}
	\cT_{\bm a + \bm b z} = \cT_u, \\
	\cM_{\bm a + \bm b z} = \cM_{v}, \\
	\cS_{\cM_{\bm a + \bm b z}} = \cS_{\cM_v}
	\end{array}
	\right \}}_{
	\text{a sub-problem of additional conditioning}
	},
\end{align*}
where $\cT_{\bm a + \bm b z}$ denotes the OT-based DA on $\bm a + \bm b z$,
$\cS_{\cM_{\bm a + \bm b z}}$ denotes a set of signs of the coefficients for the selected features in $\cM_{\bm a + \bm b z}$.
For $u \in [U], v \in [V_u]$, we aim to identify a set:
\begin{align} \label{eq:cR}
	\cR = 
	\Big \{
		(u, v) : \cM_{v} = \cM_{\rm obs}
	\Big \}.
\end{align}
The region $\cZ$ in \eq{eq:cZ} then can be identified as follows:
\begin{align}
	\cZ &= \Big \{ 
	z \in \RR 
	\mid 
	\cM_{\bm a + \bm b z} = \cM_{\rm obs}
	\Big \} \nonumber\\
	& = 
	\bigcup \limits_{(u, v) \in \cR} 
	\Bigg \{z \in \RR ~ \bigg |  
	\begin{array}{l}
	\cT_{\bm a + \bm b z} = \cT_u, \\
	\cM_{\bm a + \bm b z} = \cM_{v}, \\
	\cS_{\cM_{\bm a + \bm b z}} = \cS_{\cM_v}
	\end{array}
	\Bigg \}. \label{eq:cZ_new}
\end{align}

\textbf{Solving of each sub-problem.}
For any $u \in [U]$ and $v \in [V_u]$, we define the subset of one-dimensional projected dataset on a line for the sub-problem as:
\begin{align} \label{eq:cZ_sub_problem}
	\cZ_{u, v} = 
	\left \{z 
	~ \Big | 
	\begin{array}{l}
	\cT_{\bm a + \bm b z} = \cT_u, \\
	\cM_{\bm a + \bm b z} = \cM_{v}, 
	\cS_{\cM_{\bm a + \bm b z}} = \cS_{\cM_v}
	\end{array}
	\right \}.
\end{align}
The sub-problem region $\cZ_{u, v}$ can be re-written as:
\begin{align*}
	\cZ_{u, v}  &= \cZ_u \cap \cZ_v, \text{where }
	\cZ_u  = 
	\Big \{ 
	z \in \RR
	\mid 
	\cT_{\bm a + \bm b z} = \cT_u
	\Big \},  \\
	\cZ_v &= 
	\Big \{ 
	z \in \RR
	\mid 
	\cM_{\bm a + \bm b z} = \cM_v, \cS_{\cM_{\bm a + \bm b z}} = \cS_{\cM_v}
	\Big \}.
\end{align*}

\begin{lemma} \label{lemma:cZ_u}
The set $\cZ_u$ can be characterized by a set of quadratic inequalities w.r.t. $z$ described as follows:
\begin{align*}
	\cZ_u
	= \Big \{ 
	z \in \RR 
	\mid 
	\bm p + \bm q z + \bm r z^2 \geq \bm 0
	\Big \},
\end{align*}
where vectors $\bm p$, $\bm q$, and $\bm r$ are defined in Appendix \ref{app:proof_cZ_u}.
\end{lemma}


The proof is deferred to Appendix \ref{app:proof_cZ_u}. The purpose of Lemma \ref{lemma:cZ_u} is to ensure that the transportation $\cT_u$ remains the same for all $z \in \cZ_u$.

\begin{lemma}\label{lemma:cZ_v}
Let us define the Lasso optimization problem after the $\cT_u$ transportation as:
\begin{equation*}
	\hat{{\bm \beta}}(z) = \argmin \limits_{{\bm \beta} \in \RR^p} 
	\frac{1}{2} 
	\big \|
	\tilde{\bm Y}_u(z) - \tilde{X}_u {\bm \beta}
	\big \|^2_2 
	+ \lambda \|{\bm \beta}\|_1, 
\end{equation*}
where 
$\tilde{\bm Y}_u(z) = \Omega_u \bm Y (z)$ and $\tilde{X}_u = \Omega_u X$.
Here, 
\begin{align*}
	\Omega_u = 
	\begin{pmatrix}
		0_{n_s \times n_s} & n_s \cT_u \\
		0_{n_t \times n_s} & I_{n_t}
	\end{pmatrix}
	\in \RR^{(n_s + n_t) \times (n_s + n_t)},
\end{align*}
where $0_{n \times m} \in \RR^{n \times m}$ is the zero matrix, $I_n \in \RR^{n \times n}$ is the identity matrix, and $X = (X^s ~ X^t)^\top$. 
The set $\cZ_v$ can be identified as follows:
\begin{align*}
	\cZ_v = 
	\left \{ 
		z \in \RR ~\Bigg |
		\begin{array}{l}
			\hat{{\bm \beta}}_j(z) \neq 0, ~\forall j \in \cM_v, \\
			\hat{{\bm \beta}}_j(z) = 0, ~\forall j \not \in \cM_v, \\
			{\rm sign}\Big (\hat{{\bm \beta}}_{\cM_v}(z) \Big ) = \cS_{\cM_v}
		\end{array}
	\right \},
\end{align*}
which can be efficiently computed by solving a set of linear inequalities w.r.t $z$, derived from the Karush–Kuhn–Tucker (KKT) conditions of the Lasso.

\end{lemma}


The proof is deferred to Appendix \ref{app:proof_cZ_v}. Lemma \ref{lemma:cZ_v} guarantees that, for any $z \in \cZ_v$, the selected features and the signs of their coefficients remain the same when conducting FS after the $\cT_u$ transportation.

In Lemmas \ref{lemma:cZ_u} and \ref{lemma:cZ_v}, we demonstrate that $\cZ_u$ and $\cZ_v$ can be \emph{analytically obtained} by solving the systems of quadratic/linear inequalities, respectively.
Once $\cZ_u$ and $\cZ_v$ are computed, the sub-problem region $\cZ_{u, v}$ in \eq{eq:cZ_sub_problem} is obtained by $\cZ_{u, v}  = \cZ_u \cap \cZ_v$.

\begin{algorithm}[!t]
\caption{\texttt{SFS-DA}}
\label{alg:sfs_da}
\begin{footnotesize}
\textbf{Input:} $X^s, \bm Y^s_{\rm obs}, X^t, \bm Y^t_{\rm obs}, z_{\min}, z_{\max}$

\begin{algorithmic}[1]
\vspace{4pt}
    \STATE $\cM_{\rm obs} \gets$ FS after DA on $(X^s, \bm Y^s_{\rm obs})$ and  $(X^t, \bm Y^t_{\rm obs})$
    \vspace{4pt}
    \FOR{$j \in \cM_{\rm obs}$}
    \vspace{4pt}
        \STATE Compute $\bm \eta_{j} \gets$ Eq. (\ref{eq:eta_j}), $\boldsymbol{a} \text{ and } \boldsymbol{b} \gets$ Eq. (\ref{eq:conditional_data_space_line})
        \vspace{4pt}
        \STATE $X = (X^s ~ X^t)^\top$
        \vspace{4pt}
        \STATE $\cR \gets$ {\tt divide\_and\_conquer} (X, $\boldsymbol{a}, \boldsymbol{b}, z_{\min}, z_{\max}$)
        \vspace{4pt}
        \STATE Identify $\cZ \gets $ Eq. \eq{eq:cZ_new} with $\cR$
        \vspace{4pt}
        \STATE Compute $p^{\rm sel}_{j} \gets$ Eq. (\ref{eq:selective_p_reformulated}) with $\cZ$
        \vspace{4pt}
    \ENDFOR
\end{algorithmic}
\textbf{Output:} $\{p^{\rm sel}_{i}\}_{i \in \mathcal{M}^{\rm obs}}$
\end{footnotesize}
\end{algorithm}

\textbf{Computation of truncation region $\cZ$ by combining multiple sub-problems and algorithm.}
To identify $\cR$ in \eq{eq:cR}, the OT-based DA and Lasso FS after DA are repeatedly applied to a series of datasets
$\bm a + \bm b z$, over a sufficiently wide range off $z \in [z_{\rm min}, z_{\rm max}]$\footnote{We set $z_{\rm min} = -20\sigma$ and $z_{\rm max} = 20 \sigma$, $\sigma$ is the standard deviation of the distribution of the test statistic, because the probability mass outside this range is negligibly small.}.
For simplicity, we consider the case in which $\cZ_u$ is an interval\footnote{If $\cZ_u$ is a union of intervals, we can select the interval containing the data point that we are currently considering.}. 
Since $\cZ_v$ is also an interval, $\cZ_{u, v}$ is an interval.
We denote $\cZ_u = [\ell_u, r_u]$ and $\cZ_{u, v} = [\ell_{u, v}, r_{u, v}]$.
The divide-and-conquer procedure can be summarized in Algorithm \ref{alg:divide_and_conquer}.
After obtaining $\cR$ by Algorithm \ref{alg:divide_and_conquer}.
We can compute $\cZ$ in \eq{eq:cZ_new}, which is subsequently used to obtain the proposed selective $p$-value in \eq{eq:selective_p_reformulated}.
The entire steps of the proposed SFS-DA method is summarized in Algorithm \ref{alg:sfs_da}.

\subsection{Extension to Elastic Net} \label{sec3:extension_elastic_net}

%
In certain cases, adding an $\ell_2$ penalty to the objective function of Lasso yields the elastic net \citep{zou2005regularization}, which helps stabilize the FS results. 
Therefore, we extend our proposed method to elastic net case.
The sub-problem of FS after DA in the elastic net case is similar to that in the Lasso case, i.e., $\cZ_{u, v}  = \cZ_u \cap \cZ_v^{\rm enet}$ with 
$
\cZ_u  = 
	\Big \{ 
	z \in \RR
	\mid 
	\cT_{\bm a + \bm b z} = \cT_u
	\Big \}
$,
$
	\cZ_v^{\rm enet} = 
	\Big \{ 
	z \in \RR
	\mid 
	\cM^{\rm enet}_{\bm a + \bm b z} = \cM^{\rm enet}_v, \cS^{\rm enet}_{\cM_{\bm a + \bm b z}} = \cS^{\rm enet}_{\cM_v}
	\Big \}
$.
%
%
Here, $\cZ_u$ is the same as in the Lasso case, and $\cZ_v^{\rm enet}$ is the region corresponding to the elastic net case, whose characterization is detailed in the following lemma.

\begin{lemma}\label{lemma:cZ_v_elastic_net}
Let us define the elastic net optimization problem after the $\cT_u$ transportation as follows:
\begin{equation*}
	\hat{{\bm \beta}}^{\rm enet}(z) = \argmin \limits_{{\bm \beta} \in \RR^p} 
	\frac{1}{2} 
	\big \|
	\tilde{\bm Y}_u(z) - \tilde{X}_u {\bm \beta}
	\big \|^2_2 
	+ \lambda \|{\bm \beta}\|_1 + \frac{\gamma}{2} \|{\bm \beta}\|^2_2, 
\end{equation*}
 where $\lambda$ and $\gamma$ are the regularization parameters, 
 $\tilde{\bm Y}_u(z)$ and $\tilde{X}_u$ are defined in Lemma \ref{lemma:cZ_v}.
Then, the set $\cZ_v^{\rm enet}$ can be identified as follows:
\begin{align*}
	\cZ_v^{\rm enet} = 
	\left \{ 
		z \in \RR ~\Bigg |
		\begin{array}{l}
			\hat{{\bm \beta}}^{\rm enet}_j(z) \neq 0, ~\forall j \in \cM_v^{\rm enet}, \\
			\hat{{\bm \beta}}^{\rm enet}_j(z) = 0, ~\forall j \not \in \cM_v^{\rm enet}, \\
			{\rm sign}\Big (\hat{{\bm \beta}}^{\rm enet}_{\cM_v}(z) \Big ) = \cS_{\cM^{\rm enet}_v}
		\end{array}
	\right \},
\end{align*}
which can be efficiently computed by solving a set of linear inequalities w.r.t $z$.

\end{lemma}


The proof of Lemma \ref{lemma:cZ_v_elastic_net} is deferred to Appendix \ref{app:proof_cZ_v_elastic_net}.

\begin{algorithm}[!t]
\renewcommand{\algorithmicrequire}{\textbf{Input:}}
\renewcommand{\algorithmicensure}{\textbf{Output:}}
\begin{footnotesize}
 \begin{algorithmic}[1]
  \REQUIRE $X, \bm a, \bm b, z_{\rm min}, z_{\rm max}$
	\vspace{4pt}
	\STATE Initialization: $u = 1$, $v = 1$, $z_{u, v}  = z_{\rm min}$, $\cR = \emptyset$
	\vspace{4pt}
	\WHILE {$z_{u, v} < z_{\rm max}$}
		\vspace{4pt}
		\STATE $\cT_u \leftarrow$ DA on $\bm a + \bm b z_{u, v}$
		\vspace{4pt}
		\STATE Compute $[\ell_u, r_u] = \cZ_u \leftarrow$ Lemma \ref{lemma:cZ_u} 
		\vspace{4pt}
		\STATE $r_{u, v} = \ell_u$
		\vspace{4pt}
		\WHILE {$r_{u, v} < r_u$}
		\vspace{4pt}
		\STATE $\tilde{X}_u, \tilde{\bm Y}_u(z_{u, v}) \gets $ Lemma \ref{lemma:cZ_v}
		\vspace{4pt}
		\STATE $\cM_v$ and $\cS_{\cM_v} \leftarrow$ FS after DA on $\big (\tilde{X}_u, \tilde{\bm Y}_u(z_{u, v}) \big )$
		\vspace{4pt}
		\STATE $\cZ_v \leftarrow$ Lemma \ref{lemma:cZ_v} 
		\vspace{4pt}
		\STATE $[\ell_{u, v}, r_{u, v}] = \cZ_{u, v} \leftarrow \cZ_u \cap \cZ_v$ 
		\vspace{6pt}
		\STATE $\cR \leftarrow \cR \cup \{ (u, v)\} $ \textbf{if} $\cM_v = \cM_{\rm obs}$
		\vspace{4pt}
		\STATE $v \leftarrow v + 1$, $z_{u, v} = r_{u, v}$
		%
		%
		%
		\vspace{4pt}
		\ENDWHILE	
		\vspace{4pt}
		\STATE $v \leftarrow 1$, $u \leftarrow u + 1$, $z_{u, v} = r_{u, v}$
		\vspace{4pt}
	\ENDWHILE
	\vspace{2pt}
  \ENSURE $\cR$ 
 \end{algorithmic}
\end{footnotesize}
\caption{{\tt divide\_and\_conquer}}
\label{alg:divide_and_conquer}
\end{algorithm}

%% file: sec4.tex
\section{Experiment} \label{sec:experment}

We demonstrate the performance of the proposed SFS-DA.
Here, we present the main results. Several additional experiments can be found in Appendix \ref{app:additional_experiment}.

\subsection{Experimental Setup}

\textbf{Methods for comparison.} We compared the performance of the following methods:

$\bullet$ {\tt SFS-DA}: proposed method
    
$\bullet$ {\tt SFA-DA-oc}: proposed method, which considers only one sub-problem, i.e., over-conditioning, described in \S\ref{subsec:identification_cZ} (extension of \cite{lee2016exact} to our setting)

$\bullet$ {\tt DS:} data splitting 

$\bullet$ {\tt Bonferroni:} the most popular multiple testing

$\bullet$ {\tt Naive:} traditional statistical inference.

$\bullet$ {\tt No inference:} FS after DA without inference.

We note that if a method fails to control the FPR at  $\alpha$, it is \textit{invalid}, and its TPR becomes irrelevant. A method with a high TPR implies a low FNR.

\textbf{Synthetic data generation.} We generated $\bm Y^s$ with 
$\bm Y^s_i = {X^s_i}^\top \bm \beta^s + \veps$,
$X^s_i \sim \mathbb{N}(\bm 0, I_p), \forall i \in [n_s]$, and $\veps \sim \mathbb{N}(0, 1)$.
Similarly, $\bm Y^t$ is generated with 
$\bm Y^t_i = {X^t_i}^\top \bm \beta^t + \veps$ in which
$X^t_i \sim \mathbb{N}(\bm 0, I_p)$.
We set $p = 5$, $\lambda = 10$, $\gamma = 1$ (elastic net), and $\alpha = 0.05$. 
For the FPR experiments, all elements of $\boldsymbol{\beta}^t$ were set to 0 and $n_s \in \{50, 100, 150, 200\}$. For the TPR experiments, all elements of $\boldsymbol{\beta}^t$ were set to 0.5 and $n_s = 100$. 
We set $n_t = 10$, indicating that the target data is limited.
In all experiments, elements of $\bm \beta^s$ are set to 2. 
Note that we only conduct the inference on the target data. Therefore, the values of $\bm \beta^s$ do not affect the inference.
Each experiment was repeated 120 times.

\subsection{Numerical results}
\textbf{The results of FPRs and TPRs.}
The results of FPR and TPR in two cases of Lasso and elastic net are shown in Figs.  \ref{fig:lasso_fpr_tpr} and \ref{fig:elastic_net_fpr_tpr}. In the plots on the left, the {\tt SFS-DA}, {\tt SFS-DA-oc}, {\tt Bonferroni}, {\tt DS} controlled the FPR whereas the {\tt Naive} and {\tt No Inference} \textit{could not}. Because the {\tt Naive} and {\tt No Inference} failed to control the FPR, we no longer considered their TPRs. 
In the plots on the right, the {\tt SFS-DA} has the highest TPR compared to other methods in all the cases, i.e., the {\tt SFS-DA} has the lowest FNR.

\begin{figure}[!t]
    \centering
    \begin{subfigure}[b]{0.47\linewidth}  
        \includegraphics[width=\linewidth]{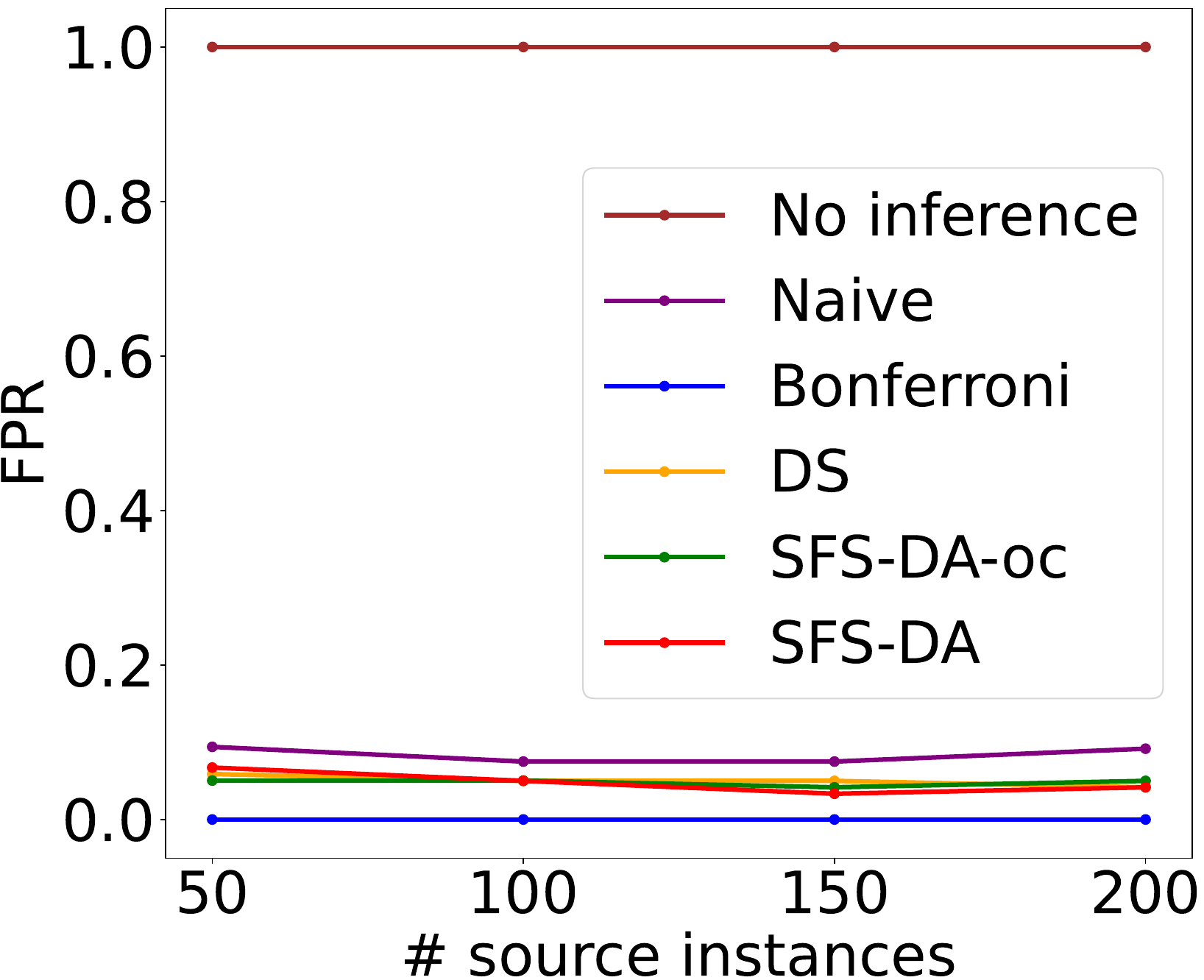}  
        \caption{FPR}
        \label{fig:fpr_univariate}
    \end{subfigure}
    \hspace{0.02\linewidth}  
    \begin{subfigure}[b]{0.47\linewidth}
        \includegraphics[width=\linewidth]{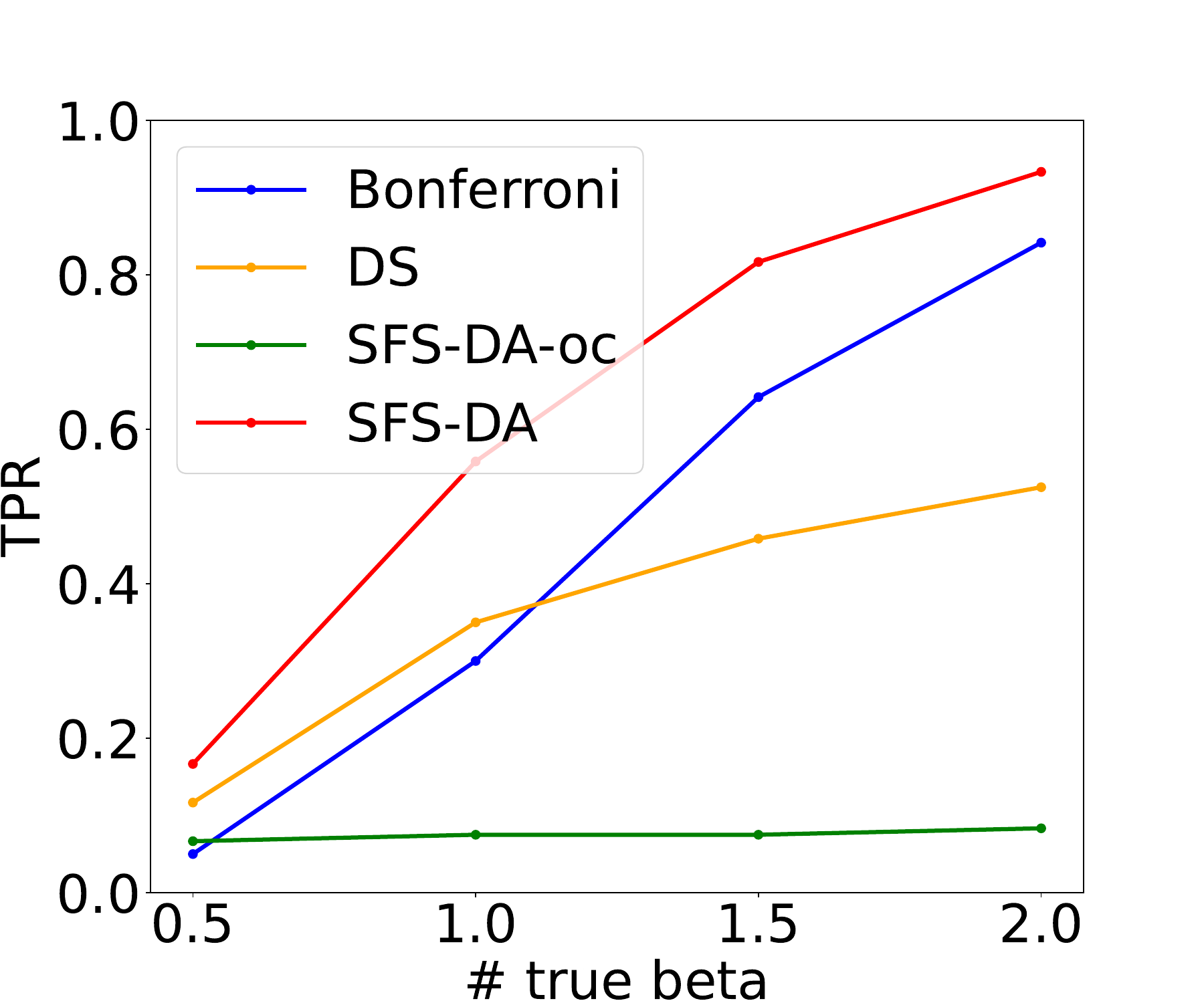}
        \caption{TPR}
        \label{fig:tpr_univariate}
    \end{subfigure}
    \vspace{-4pt}
    \caption{FPR and TPR in the case of Lasso}
    \label{fig:lasso_fpr_tpr}
    \vspace{-8pt}
\end{figure}

\begin{figure}[!t]
    \centering
    \begin{subfigure}[b]{0.47\linewidth}
        \includegraphics[width=\linewidth]{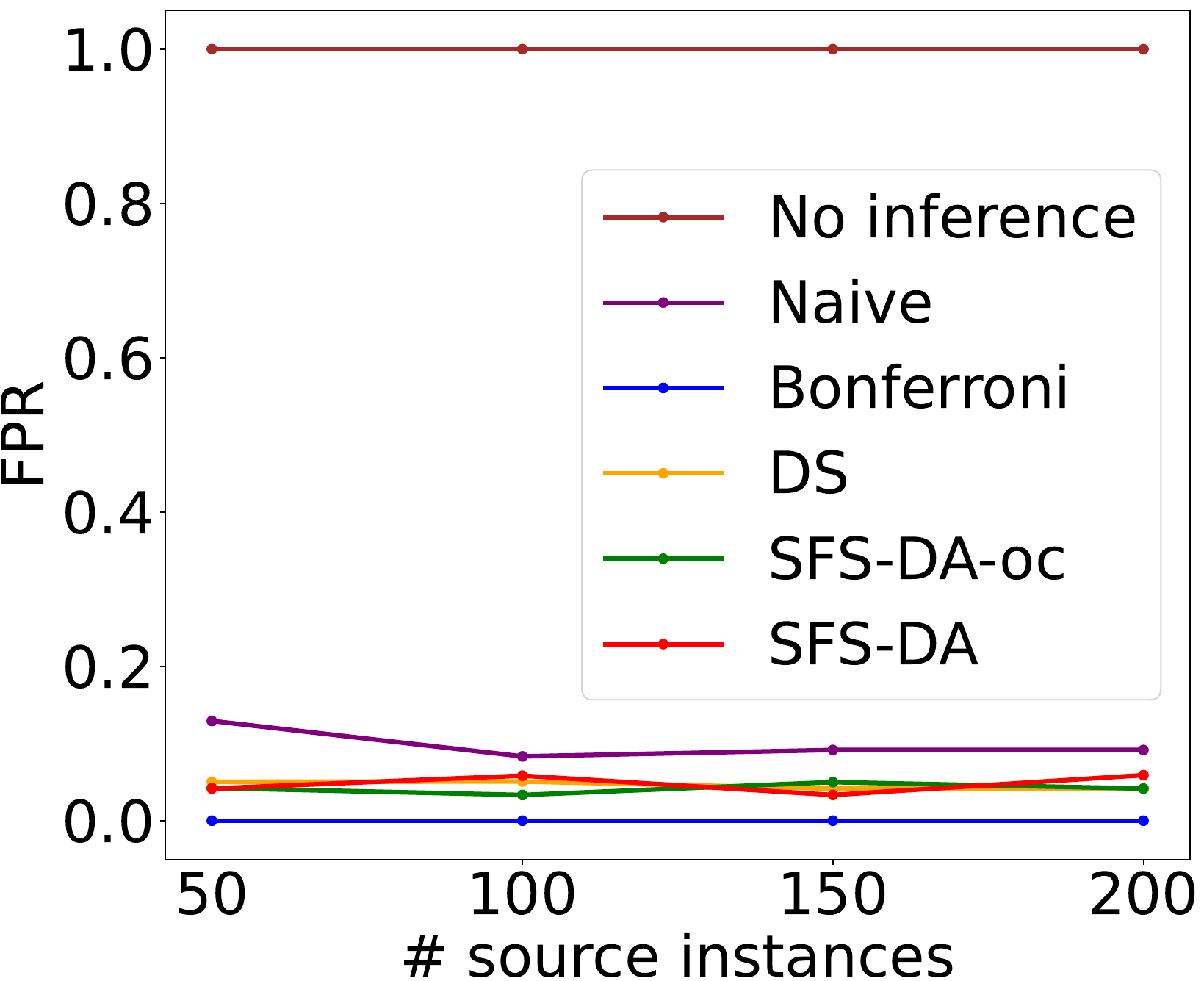}
        \caption{FPR}
        \label{fig:fpr_elastic net}
    \end{subfigure}
    \hspace{0.02\linewidth}
    \begin{subfigure}[b]{0.47\linewidth}
        \includegraphics[width=\linewidth]{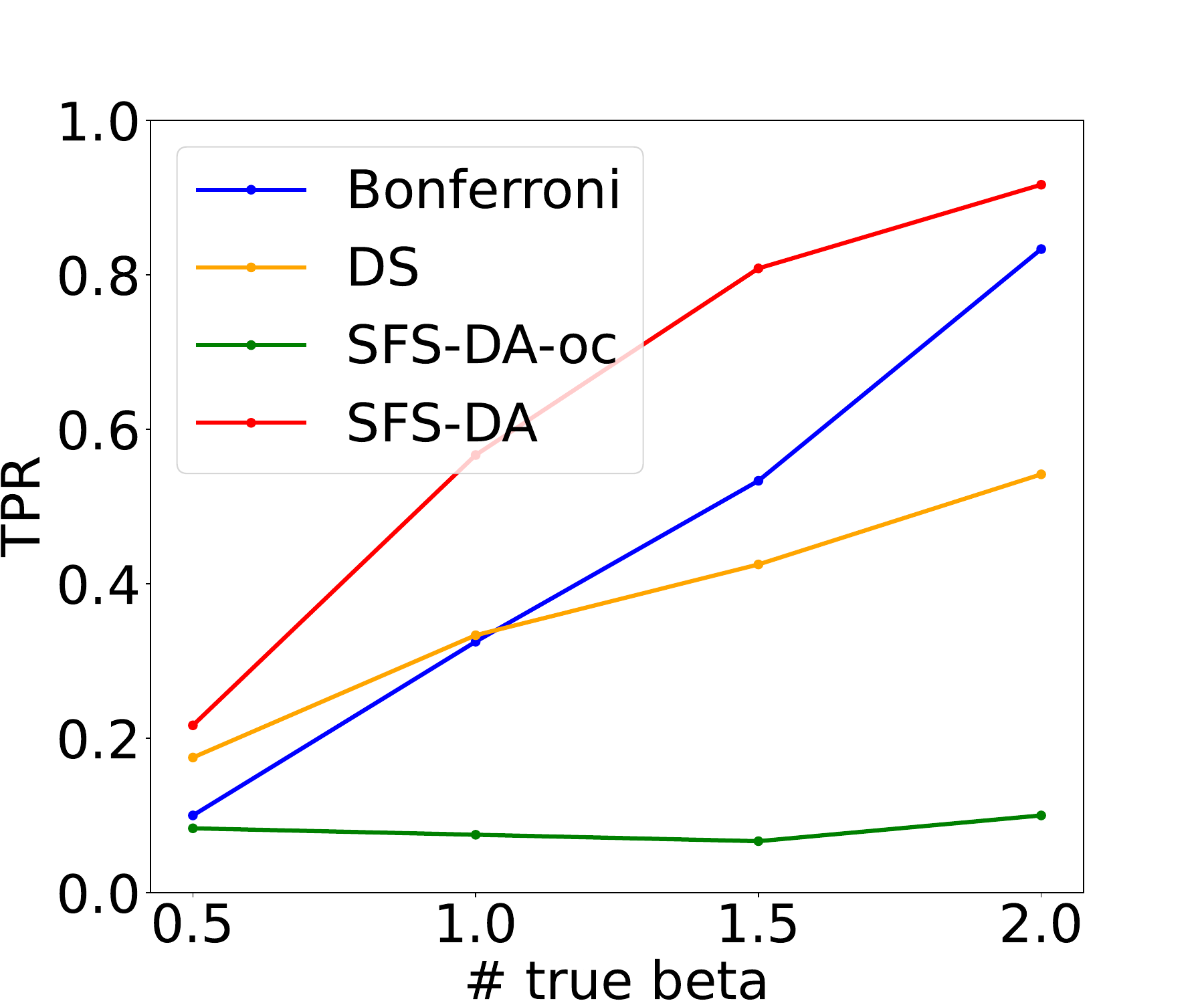}
        \caption{TPR}
        \label{fig:tpr_elastic net}
    \end{subfigure}
    \vspace{-4pt}
    \caption{FPR and TPR in the case of elastic net}
    \label{fig:elastic_net_fpr_tpr}
    \vspace{-8pt}
\end{figure}

\begin{figure}[!t]
    \centering
    \begin{subfigure}[b]{0.47\linewidth}
        \includegraphics[width=\linewidth]{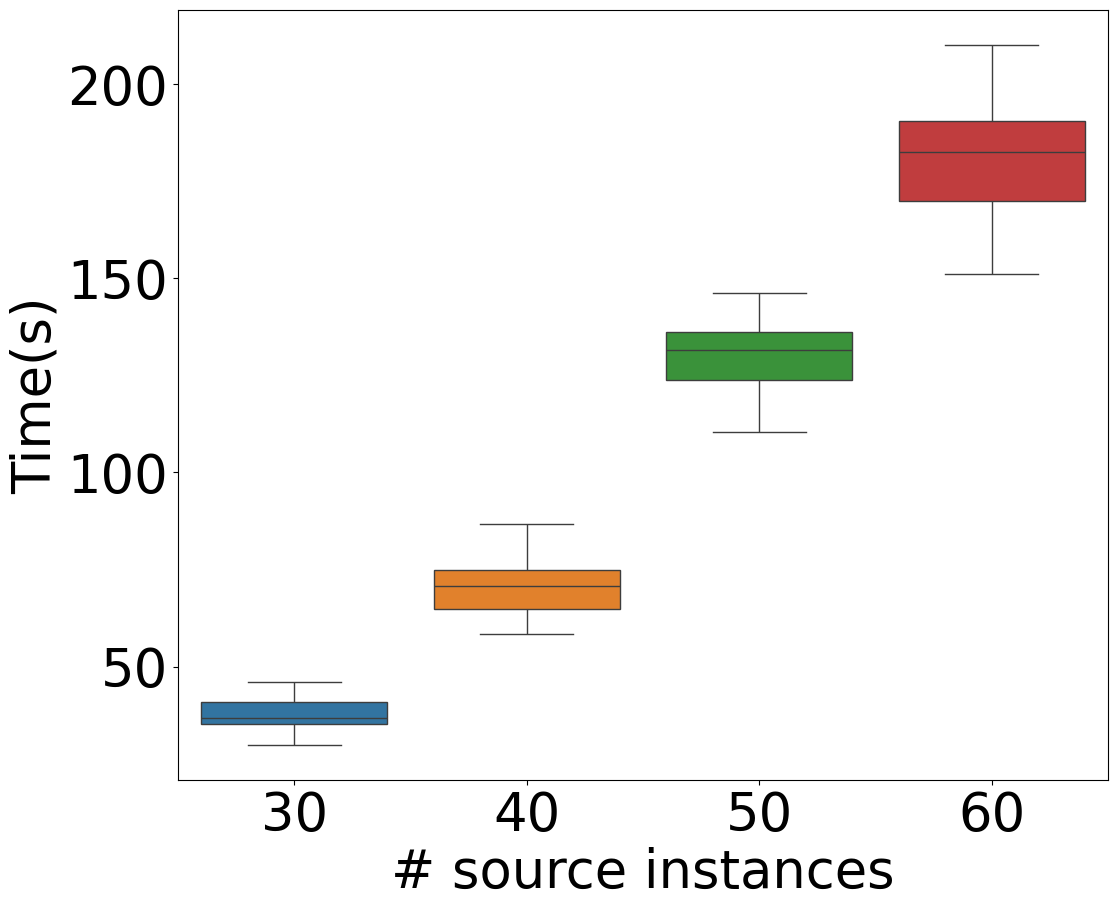}
        \caption{Computational time}
        \label{fig:time}
    \end{subfigure}
    \hspace{0.02\linewidth}
    \begin{subfigure}[b]{0.47\linewidth}
        \includegraphics[width=\linewidth]{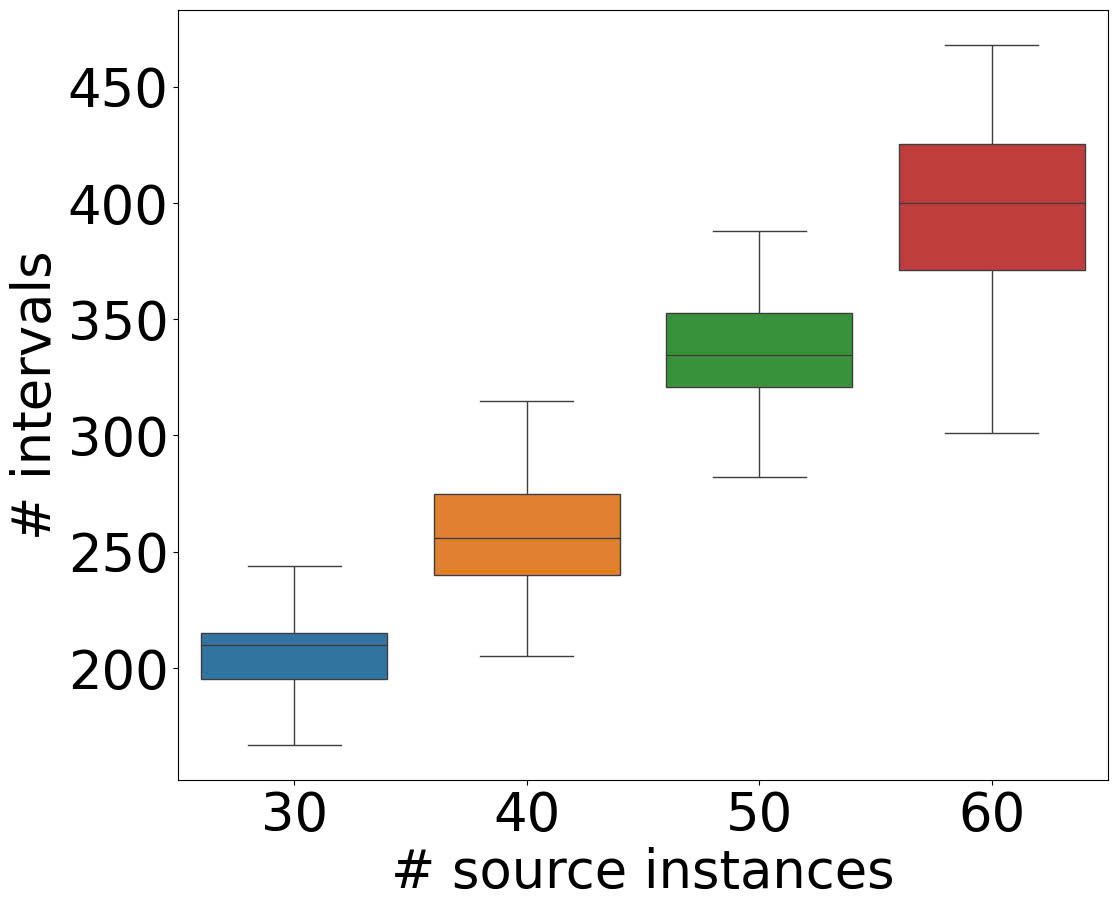}
        \caption{Encountered intervals}
        \label{fig:polytope}
    \end{subfigure}
    \vspace{-4pt}
    \caption{Computational cost of the proposed SFS-DA}
    \label{fig:computation_cost}
    \vspace{-8pt}
\end{figure}

\textbf{Computational time.} In Figure \ref{fig:computation_cost}, we show the boxplots of the time for computing each \textit{p}-value as well as actual number of intervals of $z$ that we encountered on the line when constructing the truncation region $\mathcal{Z}$ w.r.t. $n_s$. The plots demonstrate that the complexity of the {\tt SFS-DA} increases linearly w.r.t. $n_s$.

\subsection{Results on Real-World Datasets}
We performed comparison on five real-world datasets.
In this section, we present the experimental results for three datasets: the Diabetes dataset \citep{efron2004least}, the Heart Failure dataset, and the Seoul Bike dataset, all available in the UCI Machine Learning Repository. 
The results for the remaining two datasets can be found in Appendix \ref{app:additional_experiment}. 
For each dataset, we present the distribution of $p$-value for each feature. 
For each dataset, we randomly selected instances from source and target domain, with $n_s = 100$ and $n_t = 20$. 
We used Lasso for FS. 
The results are shown in Figs. \ref{fig:diabetes}, \ref{fig:heart failure}, \ref{fig:Seoul Bike}. 
The $p$-values of {\tt Bonferroni} are equal to one in almost all cases, indicating that this method is \emph{conservative}.
While the \textit{p}-value of {\tt DS} is smaller than that of {\tt SFS-DA} in a few cases (S5 in Diabetes dataset and Temperature in Seoul Bike dataset), in all remaining cases, the \textit{p}-value of the proposed {\tt SFS-DA} tends to be smaller than those of the competitors, demonstrating that {\tt SFS-DA} exhibits the highest statistical power.

\begin{figure}[!t]
    \centering
    \includegraphics[width=\linewidth]{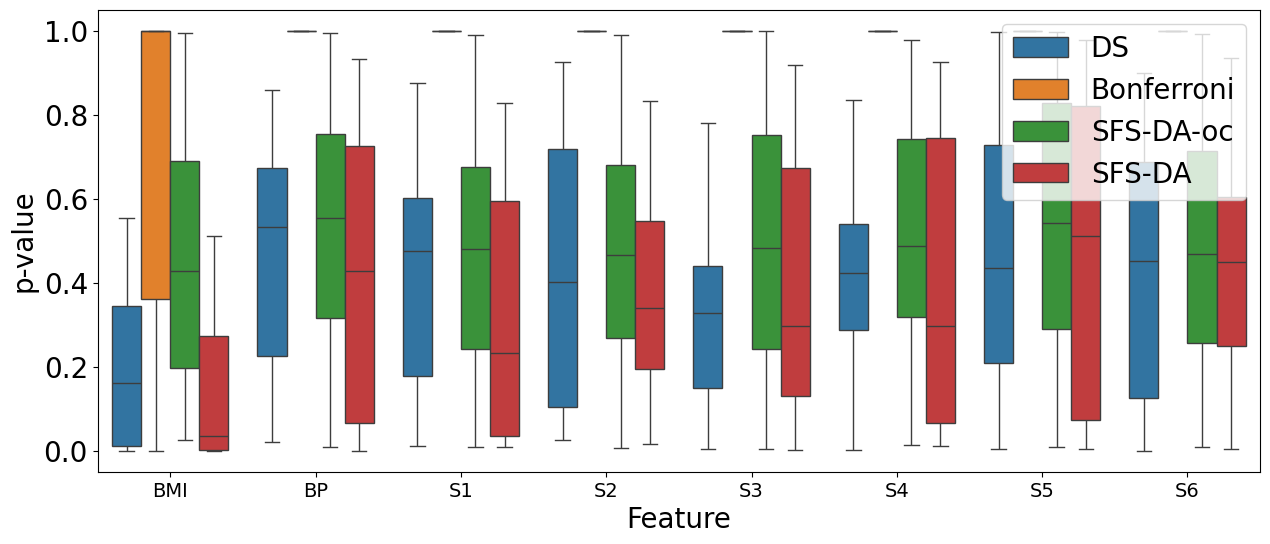}
    \vspace{-17pt}
    \caption{Diabetes dataset. The source domain consists of ``people over 50 years old'', while the target domain consists of ``people under 50 years old''.} 
    \label{fig:diabetes}
    \vspace{-8pt}
\end{figure}

\begin{figure}[!t]
    \centering
\includegraphics[width=\linewidth]{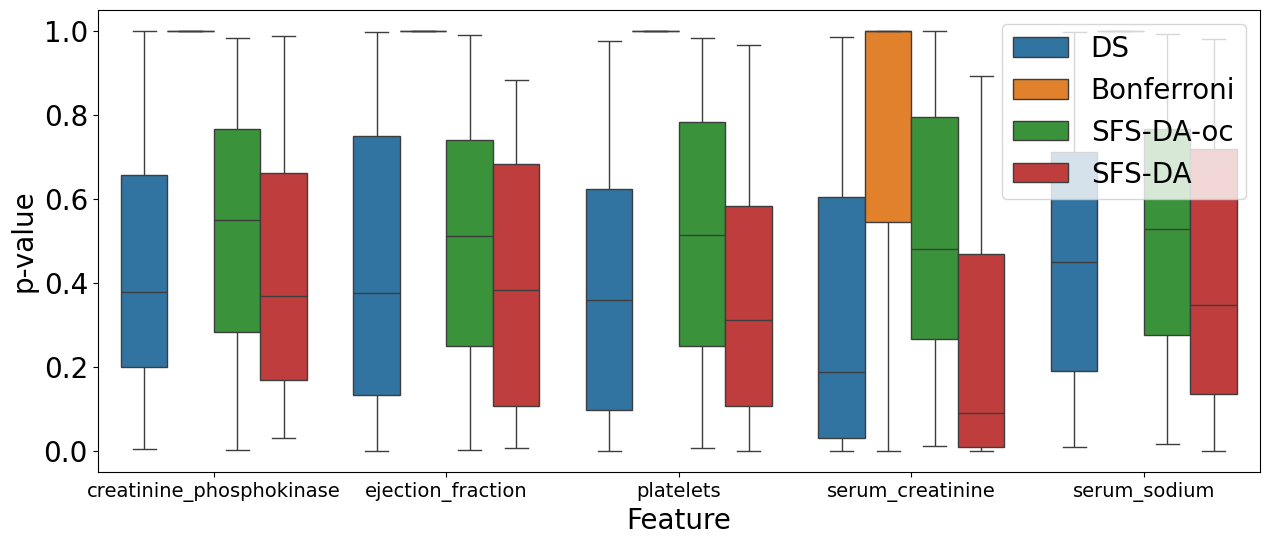}
\vspace{-17pt}
    \caption{Heart Failure dataset. The settings for the source and target domains are similar to those in Diabetes dataset.}
    \label{fig:heart failure}
    \vspace{-8pt}
\end{figure}

\begin{figure}[!t]
    \centering
    \includegraphics[width=\linewidth]{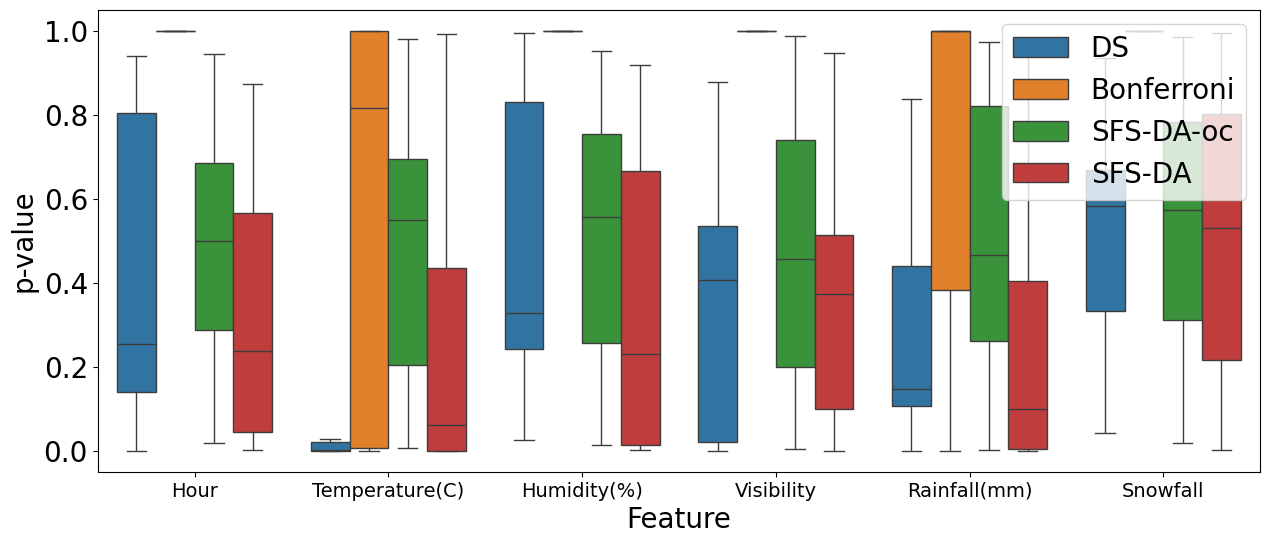}
    \vspace{-17pt}
    \caption{Seoul Bike dataset. The source domain is ``people who rent bikes on regular days'', while the target domain is ``people who rent bikes on holidays''.}
    \label{fig:Seoul Bike}
    \vspace{-15pt}
\end{figure}

%
%

%% file: sec5.tex
\section{Discussion} \label{sec:discussion}

We propose a novel setup for testing the results of FS after DA and introduce a method for computing a valid $p$-value to conduct statistical tests. 
This method leverages the SI framework and employs a divide-and-conquer approach to efficiently compute the $p$-value.
We believe this study represents a significant step toward controllable machine learning in the context of DA.
Our approach is also applicable to other feature selection algorithms where the selection event can be characterized by sets of linear or quadratic inequalities (e.g., stepwise feature selection) within the context of FS after DA. Extending the proposed method to more complex FS algorithms would represent a valuable contribution for future research.

%% file: appendix.tex
\section{Appendix}
\label{sec:appendix}

\subsection{Proof of Lemma \ref{lemma:valid_selective_p}} \label{app:proof_valid_p}

We have 
\begin{align*}
    \bm \eta_j^\top {\bm Y^s \choose \bm Y^t } \Big| \Big \{ 
	\cM_{\bm Y^s, \bm Y^t}
	=
	\cM_{\rm obs}, ~
	\cQ_{\bm Y^s, \bm Y^t}
	=
	\cQ_{\rm obs} \Big \} 
    \sim TN\left( 
            \bm \eta_j^\top {\bm \mu^s \choose \bm \mu^t
            }, \bm \eta_j^\top \Sigma \bm \eta_j, \cZ
            \right),
\end{align*}

which is a truncated normal distribution with mean ${\bm \eta}_j^\top {\bm \mu^s \choose \bm \mu^t}$, variance $\eta_j^\top \Sigma \eta_j$, in which $\Sigma = \begin{pmatrix}
	\Sigma^s & 0 \\ 
	0 & \Sigma^t
\end{pmatrix}$, and the truncation region $\cZ$ described in \S\ref{subsec:identification_cZ}. Therefore, under null hypothesis,

\begin{align*}
    p^{\rm sel}_j \ \Big| \ \Big \{ 
	\cM_{\bm Y^s, \bm Y^t}
	=
	\cM_{\rm obs}, ~
	\cQ_{\bm Y^s, \bm Y^t}
	=
	\cQ_{\rm obs} \Big \}   
    \sim \text{Unif}(0,1)
\end{align*}
Thus, $\mathbb{P}_{\rm H_{0, j}} \left( p^{\rm sel}_j \ \Big| \ 
	\cM_{\bm Y^s, \bm Y^t}
	=
	\cM_{\rm obs}, ~
	\cQ_{\bm Y^s, \bm Y^t}
	=
	\cQ_{\rm obs} 
\right) = \alpha, \forall \alpha \in [0,1]$.

Next, we have
\begin{align*}
    & \mathbb{P}_{\rm H_{0, j}} \left( p^{\rm sel}_j \ \Big| \ 
	\cM_{\bm Y^s, \bm Y^t}
	=
	\cM_{\rm obs}
\right)\\
    &= \int \mathbb{P}_{H_{0,j}} \left( p_j^{\rm sel} \leq \alpha \, \middle| \, \cM_{\bm Y^s, \bm Y^t} = \cM_{\rm obs}, \, \cQ_{\bm Y^s, \bm Y^t} = \cQ_{\rm obs} \right) \mathbb{P}_{H_{0,j}} \left( \cQ_{\bm Y^s, \bm Y^t} = \cQ_{\rm obs} \, \middle| \, \cM_{\bm Y^s, \bm Y^t} = \cM_{\rm obs} \right) \, d \cQ_{\rm obs} \\ 
    &= \int \alpha \, \mathbb{P}_{H_{0,j}} \left( \cQ_{\bm Y^s, \bm Y^t} = \cQ_{\rm obs} \, \middle| \, \cM_{\bm Y^s, \bm Y^t} = \cM_{\rm obs} \right) \, d\cQ_{\rm obs} \\ 
    &= \alpha \int \mathbb{P}_{H_{0,j}} \left( \cQ_{\bm Y^s, \bm Y^t} = \cQ_{\rm obs} \, \middle| \, \cM_{\bm Y^s, \bm Y^t} = \cM_{\rm obs} \right) \, d\cQ_{\rm obs} \\ 
    &= \alpha. 
\end{align*}

Finally, we obtain the result in Lemma \ref{lemma:valid_selective_p} as follows:
\begin{align*}
    \mathbb{P}_{\rm H_{0, j}} \left( p^{\rm sel}_j \leq \alpha
    \right) 
    &= \sum_{\cM_{\rm obs}} \mathbb{P}_{H_{0,j}} \left( p_j^{\rm sel} \leq \alpha \, \middle| \, \cM_{\bm Y^s, \bm Y^t} = \cM_{\rm obs} \right) \mathbb{P}_{H_{0,j}} \left( \cM_{\bm Y^s, \bm Y^t} = \cM_{\rm obs} \right)\\
    &= \sum_{\cM_{\rm obs}} \alpha \, \mathbb{P}_{H_{0,j}} \left( \cM_{\bm Y^s, \bm Y^t} = \cM_{\rm obs} \right) \\
    &= \alpha \sum_{\cM_{\rm obs}} \, \mathbb{P}_{H_{0,j}} \left( \cM_{\bm Y^s, \bm Y^t} = \cM_{\rm obs} \right) \\
    &= \alpha.
\end{align*}

\subsection{Proof of Lemma \ref{lemma:data_line}} \label{app:proof_line}
Base on the second condition in (\ref{eq:conditional_data_space}), we have
\begin{align*}
    \cQ_{\bm Y^s, \bm Y^t} &= \cQ_{obs} \\
    \Leftrightarrow \left( I_{n_s + n_t} - \bm b \bm \eta_j^\top \right) {\bm Y^s \choose \bm Y^t } &= \cQ_{obs} \\
    \Leftrightarrow {\bm Y^s \choose \bm Y^t } &= \cQ_{obs} + \bm b \bm \eta_j^\top  {\bm Y^s \choose \bm Y^t }.
\end{align*}
By defining $\bm a= \cQ_{obs}, z = \bm \eta_j^\top{\bm Y^s \choose \bm Y^t }$, and incorporating the second condition of (\ref{eq:conditional_data_space}), we obtain Lemma \ref{lemma:data_line}.

\subsection{Proof of Lemma \ref{lemma:cZ_u}} \label{app:proof_cZ_u}
The proof is constructed based on the results presented in \cite{le2024cad}, in which the authors introduced an approach to characterize the event of OT by using the concept of \textit{parametric linear programming}. Let us re-written the OT problem between the source and target domain in (\ref{eq:ot_problem}) as:
\[
    \hat{\boldsymbol{t}} = \underset{\boldsymbol{t} \in \mathbb{R}^{n_s  n_t}} {\arg\min}  \ \boldsymbol{t}^\top \boldsymbol{c} \left(D^s, D^t\right) 
\]
\[
    \qquad \text{s.t.} \quad H \boldsymbol{t} = \boldsymbol{h}, \boldsymbol{t} \geq 0,
\]

where $\boldsymbol{t} =  \text{vec}(T), \boldsymbol{c} \left(D^s, D^t\right) = \text{vec} \left(C \left(D^s, D^t\right) \right) = {\bm c'} +
\left[ \Theta \begin{pmatrix}
    \bm Y^s \\
    \bm Y^t
\end{pmatrix} \right] \circ \left[ \Theta \begin{pmatrix}
    \bm Y^s \\
    \bm Y^t
\end{pmatrix} \right]$,
\[
    {\bm c'} = \text{vec} \left(\Big[
	\big \| X_i^s - X_j^t \big \|^2_2 
	\Big]_{ij}\right) \in \RR^{n_sn_t},
\]
\[
    \Theta = \text{hstack}\left(I_{n_s} \otimes \mathbf{1}_{n_t}, - \mathbf{1}_{n_s} \otimes I_{n_t}\right) \in \mathbb{R}^{n_sn_t \times (n_s + n_t)},
\]
the cost vector ${\bm c'}$ once computed from $X^s$ and $X^t$ remains fixed, vec($\cdot$) is an operator that transforms a matrix into a vector with concatenated rows, the operator $\circ$ is element-wise product, hstack($\cdot,\cdot$) is horizontal stack operation, the operator $\otimes$ is Kronecker product, $I_{n} \in \mathbb{R}^{n \times n}$ is the identity matrix, and $\mathbf{1}_{m} \in \mathbb{R}^{m}$ is a vector of ones. The matrix $H$ is defined  as $H = \begin{pmatrix}
    H_r & H_c
\end{pmatrix}^\top \in \mathbb{R}^{(n_s+n_t) \times n_sn_t}$ in which

\[
H_r = 
\begin{bmatrix}
1 & \dots & 1 & 0 & \dots & 0 & \dots & 0 & \dots & 0 \\
0 & \dots & 0 & 1 & \dots & 1 & \dots & 0 & \dots & 0 \\
\vdots & \dots & \vdots & \vdots & \dots & \vdots & \dots & \vdots & \dots & \vdots \\
0 & \dots & 0 & 0 & \dots & 0 & \dots & 1 & \dots & 1 \\
\end{bmatrix}
\in \mathbb{R}^{n_s \times n_s n_t}
\]

that performs the sum over the rows of $T$ and
\[
    H_c = \begin{bmatrix}
I_{n_t} & I_{n_t} & \dots & I_{n_t}
\end{bmatrix} \in \mathbb{R}^{n_t \times n_s n_t}
\]
that performs the sum over the columns of $T$, and $h = \left(\frac{\mathbf{1}_{n_s}}{n_s}, \frac{\mathbf{1}_{n_t}}{n_t}\right)^\top \in \mathbb{R}^{n_s + n_t}.$

Next, we consider the OT problem with the parametrized data $\boldsymbol{a} + \boldsymbol{b}z:$

\begin{align*}
    &\min_{{\bm t} \in \mathbb{R}^{n_sn_t}} {\bm t}^T \left[ ({\bm c'} + \Theta(\boldsymbol{a} + \boldsymbol{b}z) ) \circ ({\bm c'} + \Theta(\boldsymbol{a} + \boldsymbol{b}z) ) \right] \quad \text{s.t.} \quad H{\bm t} = h, \quad {\bm t} \geq 0, \\
    \Leftrightarrow &\min_{{\bm t} \in \mathbb{R}^{n_sn_t}} (\tilde{\bm p} + \tilde{\bm q}z + \tilde{\bm r}z^2)^\top {\bm t} \quad \text{s.t.} \quad H{\bm t} = h, \quad {\bm t} \geq 0.
\end{align*}

where
\[
    \tilde{\boldsymbol{p}} = ({\bm c'} + \Theta \boldsymbol{a}) \circ ({\bm c'} + \Theta \boldsymbol{a}), \quad 
    \tilde{\boldsymbol{q}} = (\Theta \boldsymbol{a}) \circ (\Theta \boldsymbol{b}) + (\Theta \boldsymbol{b}) \circ (\Theta \boldsymbol{a}), \quad \text{and} \quad \tilde{\boldsymbol{r}} = (\Theta \boldsymbol{b}) \circ (\Theta \boldsymbol{b}).
\]
By fixing $\mathcal{B}_u$ as the optimal basic index set of the linear program, the \textit{relative cost vector} w.r.t to the set of non-basis variables $\mathcal{B}_u^c$ is defined as
\[
    \boldsymbol{r}_{\mathcal{B}_u^c} = \boldsymbol{p} + \boldsymbol{q} z + \boldsymbol{r}z^2,
\]
where
\begin{equation}
    \label{define p q r}
    \boldsymbol{p} = (\tilde{\boldsymbol{p}}_{\mathcal{B}_u^c}^\top - \tilde{\boldsymbol{p}}_{\mathcal{B}_u}^\top H_{:,\mathcal{B}_u}^{-1} H_{:,\mathcal{B}_u^c})^\top, \quad
    \boldsymbol{q} = (\tilde{\boldsymbol{q}}_{\mathcal{B}_u^c}^\top - \tilde{\boldsymbol{q}}_{\mathcal{B}_u}^\top H_{:,\mathcal{B}_u}^{-1} H_{:,\mathcal{B}_u^c}) ^\top, \quad
    \boldsymbol{r} = (\tilde{\boldsymbol{r}}_{\mathcal{B}_u^c}^\top - \tilde{\boldsymbol{r}}_{\mathcal{B}_u}^\top H_{:,\mathcal{B}_u}^{-1} H_{:,\mathcal{B}_u^c})^\top,    
\end{equation}
$H_{:,\mathcal{B}_u}^{-1}$ is a sub-matrix of $H$ made up of all rows and columns in the set $\mathcal{B}_u$. The requirement for $\mathcal{B}_u$ to be the optimal basis index set is $\boldsymbol{r}_{\mathcal{B}_u^c} \geq \boldsymbol{0}$ (i.e., the cost in minimization problem will never decrease when the non-basic variables become positive and enter the basis). We note that the optimal basis index set $\mathcal{B}_u$ corresponds to the transportation $\mathcal{T}_u$. Therefore, the set $\mathcal{Z}_u$ is defined as 
\[
\begin{aligned}
    \mathcal{Z}_u &= \{ z \in \mathbb{R} \mid \mathcal{T}_{{\bm a} + {\bm b} z} = \mathcal{T}_u \}, \\
    &= \{ z \in \mathbb{R} \mid \mathcal{B}_{{\bm a} + {\bm b} z} = \mathcal{B}_u \}, \\
    &= \{ z \in \mathbb{R} \mid \boldsymbol{r}_{\mathcal{B}_u^c} = \boldsymbol{p} + \boldsymbol{q} z + \boldsymbol{r}z^2 \geq 0 \}.
\end{aligned}
\]
Thus, we obtain the result in Lemma \ref{lemma:cZ_u}.

\subsection{Proof of Lemma \ref{lemma:cZ_v}} \label{app:proof_cZ_v}
The identification of $\mathcal{Z}_v$ is constructed based on the results presented in \cite{lee2016exact}, in which the authors characterized conditioning event of Lasso by deriving from the KKT conditions. Let us define the KKT conditions of the Lasso after the $\mathcal{T}_u$ transportation as following:
\begin{equation}
\begin{aligned}
\label{eq: kkt conditions}
        \tilde{X}^\top_u \big(\tilde{X}_u \hat{\boldsymbol{\beta}}(z) &- \tilde{\bm Y}_u(z)\big) + \lambda \cS = 0, \\
        \cS_j &= \operatorname{sign}(\hat{\boldsymbol{\beta}}_j(z)), \quad \text{if } \hat{\boldsymbol{\beta}}_j(z) \neq 0, \\
        \cS_j &\in (-1, 1), \quad\quad\quad \text{if } \hat{\boldsymbol{\beta}}_j(z) = 0.
\end{aligned}
\end{equation}
The two first conditions of the set: 
\[\cZ_v = \left \{ 
		z \in \RR ~\Bigg |
		\begin{array}{l}
			\hat{{\bm \beta}}_j(z) \neq 0, ~\forall j \in \cM_v, \\
			\hat{{\bm \beta}}_j(z) = 0, ~\forall j \not \in \cM_v, \\
			{\rm sign}\Big (\hat{{\bm \beta}}_{\cM_v}(z) \Big ) = \cS_{\cM_v}
		\end{array}
	\right \}\] lead to the set $\cM_v$ being the result of the Lasso after DA. Then, by partitioning Eq. (\ref{eq: kkt conditions}) according to the active set $\cM_v$, adopting the convention that $\cM_{v_c}$ means "$\cM_v$'s complement", the KKT conditions in (\ref{eq: kkt conditions}) can be rewritten as following:
\begin{equation}
\label{eq: partition KKT conditions}
    \begin{aligned}
        \tilde{X}_{u_{\mathcal{M}_v}}^\top \left(\tilde{X}_{u_{\mathcal{M}_v}} \hat{\boldsymbol{\beta}}_{\mathcal{M}_v}(z) - \tilde{\bm Y}_u(z)\right) + \lambda \mathcal{S}_{\mathcal{M}_v} &= 0, \\
        \tilde{X}_{u_{\mathcal{M}_{v_c}}}^\top \left(\tilde{X}_{u_{\mathcal{M}_v}} \hat{\boldsymbol{\beta}}_{\mathcal{M}_v}(z) - \tilde{\bm Y}_u(z)\right) + \lambda \mathcal{S}_{\mathcal{M}_{v_c}} &= 0,\\
        \operatorname{sign}(\hat{\boldsymbol{\beta}}_{\mathcal{M}_v}(z)) &= \mathcal{S}_{\mathcal{M}_v}, \\
        \left\| \mathcal{S}_{\mathcal{M}_{v_c}} \right\|_\infty &< \mathbf{1}.
    \end{aligned}
\end{equation}
By solving the first two equations (\ref{eq: partition KKT conditions}) for $\hat{\boldsymbol{\beta}}_{\mathcal{M}_v}(z)$ and $\mathcal{S}_{\mathcal{M}_{v_c}}$, we obtain the equivalent set of conditions:
\begin{equation}
\begin{aligned}
    \hat{\bm \beta}_{\cM_v}(z) &= (\tilde{X}_{u_{\cM_v}}^\top \tilde{X}_{u_{\cM_v}})^{-1} (\tilde{X}_{u_{\cM_v}}^\top \tilde{\bm Y}_u(z) - \lambda\mathcal{S}_{\cM_v}), \\
    \cS_{\cM_{v_c}} &= \tilde{X}_{u_{\cM_{v_c}}}^\top (\tilde{X}_{u_{\cM_v}}^\top)^+ \mathcal{S}_{\cM_{v}} + \frac{1}{\lambda} \tilde{X}_{u_{\cM_{v_c}}}^\top (I_{n_s+n_t} - \tilde{X}_{u_{\cM_v}}(\tilde{X}_{u_{\cM_v}})^+)\tilde{\bm Y}_u(z), \\
    \operatorname{sign}(\hat{\bm \beta}_{\cM_v}(z)) &= \mathcal{S}_{\cM_v}, \\
    \left\|\cS_{\cM_{v_c}}\right\|_\infty &< \mathbf{1},
\end{aligned}
\nonumber
\end{equation}
where $\left(X\right)^+ = (X^\top X)^{-1}X^\top$, $\left(X^\top\right)^+ = X(X^\top X)^{-1}$. 
Then, the set $\cZ_v$ can be rewritten as:
\begin{equation}
\begin{aligned}
    \cZ_v &= \left \{ 
		z \in \RR \Bigg|
		\begin{array}{l}
        \begin{aligned}
			(\tilde{X}_{u_{\cM_v}}^\top \tilde{X}_{u_{\cM_v}})^{-1} (\tilde{X}_{u_{\cM_v}}^\top \tilde{\bm Y}_u(z) - \lambda\mathcal{S}_{\cM_v}) &= \hat{\bm \beta}_{\cM_v}(z), \\
    \tilde{X}_{u_{\cM_{v_c}}}^\top (\tilde{X}_{u_{\cM_v}}^\top)^+ \mathcal{S}_{\cM_{v}} + \frac{1}{\lambda} \tilde{X}_{u_{\cM_{v_c}}}^\top (I_{n_s+n_t} - \tilde{X}_{u_{\cM_v}}(\tilde{X}_{u_{\cM_v}})^+)\tilde{\bm Y}_u(z) &= \cS_{\cM_{v_c}},\\
    \operatorname{sign}(\hat{\bm \beta}_{\cM_v}(z)) &= \mathcal{S}_{\cM_v}, \\
    \left\|\cS_{\cM_{v_c}}\right\|_\infty &< \mathbf{1}.
	\end{aligned}
    \end{array}
    \right \}
\nonumber
\end{aligned}
\end{equation}
The two last conditions of $\cZ_v$ then can be rewritten as:
\begin{equation}
\label{sign condition}
\begin{aligned}
    &\quad\left\{\operatorname{sign}(\hat{\bm \beta}_{\cM_v}(z)) = \mathcal{S}_{\cM_v} \right\} \\ 
    &= \left\{\cS_{\cM_v} \circ \hat{\bm \beta}_{\cM_v}(z) > \mathbf{0}\right\}, \\
    &= \left\{ \cS_{\cM_v} \circ (\tilde{X}_{u_{\cM_v}}^\top \tilde{X}_{u_{\cM_v}})^{-1} (\tilde{X}_{u_{\cM_v}}^\top \tilde{\bm Y}_u(z) - \lambda\mathcal{S}_{\cM_v}) > \mathbf{0} \right\}, \\
    &= \left\{ \mathcal{S}_{\mathcal{M}_v} \circ \left(\tilde{X}_{u_{\cM_v}}\right)^+\tilde{\bm Y}_u(z) > \lambda \mathcal{S}_{\cM_v} \circ \left( \left(\tilde{X}_{u_{\cM_v}}^\top \tilde{X}_{u_{\cM_v}}\right)^{-1} \mathcal{S}_{\mathcal{M}_v}\right)  \right\}, \\
    &= \left\{ \mathcal{S}_{\mathcal{M}_v} \circ \left(\tilde{X}_{u_{\cM_v}}\right)^+\Omega_u {\bm Y}(z) > \lambda \mathcal{S}_{\cM_v} \circ \left( \left(\tilde{X}_{u_{\cM_v}}^\top \tilde{X}_{u_{\cM_v}}\right)^{-1} \mathcal{S}_{\mathcal{M}_v}\right)  \right\}, \\
    &= \left\{ \mathcal{S}_{\mathcal{M}_v} \circ \left(\tilde{X}_{u_{\cM_v}}\right)^+\Omega_u (\bm a + \bm b z) > \lambda \mathcal{S}_{\cM_v} \circ \left( \left(\tilde{X}_{u_{\cM_v}}^\top \tilde{X}_{u_{\cM_v}}\right)^{-1} \mathcal{S}_{\mathcal{M}_v}\right)  \right\}, \\
    &= \left\{ {\bm \psi}_0 z \leq {\bm \phi}_0\right\},
    \end{aligned}
    \nonumber
    \end{equation}
    \begin{equation}
    \begin{aligned}
    &\quad\left\{\left\|\cS_{\cM_{v_c}}\right\|_\infty < \mathbf{1}\right\} \\
    &= \left\{-\mathbf{1} < \cS_{\cM_{v_c}} < \mathbf{1}\right\}, \\
    &= \left\{ -\mathbf{1} < \tilde{X}_{u_{\cM_{v_c}}}^\top (\tilde{X}_{u_{\cM_v}}^\top)^+ \mathcal{S}_{\cM_{v}} + \frac{1}{\lambda} \tilde{X}_{u_{\cM_{v_c}}}^\top (I_{n_s+n_t} - \tilde{X}_{u_{\cM_v}}(\tilde{X}_{u_{\cM_v}})^+)\tilde{\bm Y}_u(z) < \mathbf{1} \right\}, \\
    &= \left\{ \begin{aligned}
        \frac{1}{\lambda} \tilde{X}_{u_{\cM_{v_c}}}^\top (I_{n_s+n_t} - \tilde{X}_{u_{\cM_v}}(\tilde{X}_{u_{\cM_v}})^+)\Omega_u\bm Y(z) &< \mathbf{1} - \tilde{X}_{u_{\cM_{v_c}}}^\top (\tilde{X}_{u_{\cM_v}}^\top)^+ \mathcal{S}_{\cM_v} \\
        - \frac{1}{\lambda} \tilde{X}_{u_{\cM_{v_c}}}^\top (I_{n_s+n_t} - \tilde{X}_{u_{\cM_v}}(\tilde{X}_{u_{\cM_v}})^+)\Omega_u\bm Y(z) &< \mathbf{1} + \tilde{X}_{u_{\cM_{v_c}}}^\top (\tilde{X}_{u_{\cM_v}}^\top)^+ \mathcal{S}_{\cM_v}
    \end{aligned}
    \right\}, \\
    &= \left\{ \begin{aligned}
        \frac{1}{\lambda} \tilde{X}_{u_{\cM_{v_c}}}^\top (I_{n_s+n_t} - \tilde{X}_{u_{\cM_v}}(\tilde{X}_{u_{\cM_v}})^+)\Omega_u\bm (\bm a + \bm b z) &< \mathbf{1} - \tilde{X}_{u_{\cM_{v_c}}}^\top (\tilde{X}_{u_{\cM_v}}^\top)^+ \mathcal{S}_{\cM_v} \\
        - \frac{1}{\lambda} \tilde{X}_{u_{\cM_{v_c}}}^\top (I_{n_s+n_t} - \tilde{X}_{u_{\cM_v}}(\tilde{X}_{u_{\cM_v}})^+)\Omega_u\bm (\bm a + \bm b z) &< \mathbf{1} + \tilde{X}_{u_{\cM_{v_c}}}^\top (\tilde{X}_{u_{\cM_v}}^\top)^+ \mathcal{S}_{\cM_v}
    \end{aligned}
    \right\}, \\
    &= \left\{ 
    \begin{pmatrix}
        {\bm \psi}_{10} \\
        {\bm \psi}_{11}
    \end{pmatrix} z \leq 
    \begin{pmatrix}
        {\bm \phi}_{10} \\
        {\bm \phi}_{11}
    \end{pmatrix}\right\}, \\
    &= \left\{ {\bm \psi}_1 z \leq {\bm \phi}_1 \right\}.
\end{aligned}
\nonumber
\end{equation}
Finally, the set $\mathcal{Z}_v$ can be defined as:
\begin{equation}
\begin{aligned}
    \mathcal{Z}_v &= \left\{z \in \mathbb{R} \mid {\bm \psi} z \leq {\bm \phi} \right\},
\end{aligned}
\nonumber
\end{equation}
where ${\bm \psi} = \left({\bm \psi}_0 \quad {\bm \psi}_1\right)^\top$, ${\bm \phi} = \left({\bm \phi}_0 \quad {\bm \phi}_1\right)^\top$. \newline
Thus, the set $\cZ_v$ can be identified by solving a set of linear inequalities w.r.t $z$.
\subsection{Proof of Lemma \ref{lemma:cZ_v_elastic_net}} \label{app:proof_cZ_v_elastic_net}
The identification of $\cZ_v^{\rm enet}$ is constructed similar to that in Lemma \ref{lemma:cZ_v}. Let us define the KKT conditions of the elastic net after the $\mathcal{T}_u$ transportation as following:
\begin{equation}
\begin{aligned}
\label{eq: kkt conditions elastic net}
    (\tilde{X}_u^\top \tilde{X}_u + \gamma I_p) \hat{\boldsymbol{\beta}}^{\rm enet}(z) &- \tilde{X}_u^\top \tilde{\bm Y}_u(z) + \lambda \cS = 0, \\
    \cS_j &= \operatorname{sign}(\hat{\boldsymbol{\beta}}^{\rm enet}_j(z)), \quad \text{if } \hat{\boldsymbol{\beta}}^{\rm enet}_j(z) \neq 0, \\
    \cS_j &\in (-1, 1), \quad\quad\quad \text{if } \hat{\boldsymbol{\beta}}^{\rm enet}_j(z) = 0.
\end{aligned}
\end{equation}
By following the same approach as in Lemma \ref{lemma:cZ_v}, the KKT conditions of elastic net in (\ref{eq: kkt conditions elastic net}) can be partitioned according to $\cM_v^{\rm enet}$ as below:
\begin{equation}
\label{eq: partition KKT conditions elastic net}
    \begin{aligned}        (\tilde{X}_{u_{\cM_v^{\rm enet}}}^\top \tilde{X}_{u_{\cM_v^{\rm enet}}} + \gamma I) \hat{\boldsymbol{\beta}}^{\rm enet}(z) - \tilde{X}_{u_{\cM_v^{\rm enet}}}^\top \tilde{\bm Y}_u(z) + \lambda \cS &= 0, \\
        \tilde{X}_{u_{\mathcal{M}_{v_c}^{\rm enet}}}^\top \left(\tilde{X}_{u_{\cM_v^{\rm enet}}} \hat{\boldsymbol{\beta}}_{\cM_v^{\rm enet}}(z) - \tilde{\bm Y}_u(z)\right) + \lambda \mathcal{S}_{\cM_{v_c}^{\rm enet}} &= 0,\\
        \operatorname{sign}(\hat{\boldsymbol{\beta}}_{\mathcal{M}_v}(z)) &= \mathcal{S}_{\mathcal{M}_v}, \\
        \left\| \mathcal{S}_{\mathcal{M}_{v_c}} \right\|_\infty &< \mathbf{1}.
    \end{aligned}
\end{equation}
By solving the first two equations (\ref{eq: partition KKT conditions elastic net}) for $\hat{\boldsymbol{\beta}}_{\mathcal{M}_v^{\rm enet}}(z)$ and $\mathcal{S}_{\mathcal{M}_{v_c}^{\rm enet}}$, we obtain the equivalent set of conditions:
\begin{equation}
\label{rewrite KKT conditions elastic net}
\begin{aligned}
    \hat{\boldsymbol{\beta}}_{\mathcal{M}_v^{\rm enet}}(z) &= (\tilde{X}_{u_{\cM_v^{\rm enet}}}^\top \tilde{X}_{u_{\cM_v^{\rm enet}}} + \gamma I)^{-1} (\tilde{X}_{u_{\cM_v^{\rm enet}}}^\top \tilde{\bm Y}_u(z) - \lambda\mathcal{S}_{\mathcal{M}^{\rm enet}_v}), \\
    \mathcal{S}_{\mathcal{M}_{v_c}^{\rm enet}} &= \tilde{X}_{u_{\cM_{v_c}^{\rm enet}}}^\top (\tilde{X}_{u_{\cM_v^{\rm enet}}})^{**} \mathcal{S}_{\mathcal{M}^{\rm enet}_v} + \frac{1}{\lambda} \tilde{X}_{u_{\cM_{v_c}^{\rm enet}}}^\top (I_{n_s+n_t} - \tilde{X}_{u_{\cM_v^{\rm enet}}}(\tilde{X}_{u_{\cM_v^{\rm enet}}})^*)\tilde{\bm Y}_u(z), \\
    \operatorname{sign}(\hat{\boldsymbol{\beta}}_{\mathcal{M}_v^{\rm enet}}(z)) &= \mathcal{S}_{\mathcal{M}^{\rm enet}_v}, \\
    \left\|\mathcal{S}_{\mathcal{M}_{v_c}^{\rm enet}}\right\|_\infty &< \mathbf{1},
\end{aligned}
\nonumber
\end{equation}
where $(X)^* = (X^\top X + \gamma I)^{-1}X^\top$, $(X)^{**} = X(X^\top X + \gamma I)^{-1}$. The set $\cZ_v^{\rm enet}$ can be rewritten as:
\begin{equation}
\begin{aligned}
    \cZ_v^{\rm enet} &= \left \{ 
		z \in \RR \Bigg|
		\begin{array}{l}
        \begin{aligned}
    (\tilde{X}_{u_{\cM_v^{\rm enet}}}^\top \tilde{X}_{u_{\cM_v^{\rm enet}}} + \gamma I)^{-1} (\tilde{X}_{u_{\cM_v^{\rm enet}}}^\top \tilde{\bm Y}_u(z) - \lambda\mathcal{S}_{\mathcal{M}^{\rm enet}_v}) &= \hat{\boldsymbol{\beta}}_{\mathcal{M}_v^{\rm enet}}(z), \\
    \tilde{X}_{u_{\cM_{v_c}^{\rm enet}}}^\top (\tilde{X}_{u_{\cM_v^{\rm enet}}})^{**} \mathcal{S}_{\mathcal{M}^{\rm enet}_v} + \frac{1}{\lambda} \tilde{X}_{u_{\cM_{v_c}^{\rm enet}}}^\top (I_{n_s+n_t} - \tilde{X}_{u_{\cM_v^{\rm enet}}}(\tilde{X}_{u_{\cM_v^{\rm enet}}})^*)\tilde{\bm Y}_u(z) &= \mathcal{S}_{\mathcal{M}_{v_c}^{\rm enet}}, \\
    \operatorname{sign}(\hat{\boldsymbol{\beta}}_{\mathcal{M}_v^{\rm enet}}(z)) &= \mathcal{S}_{\mathcal{M}^{\rm enet}_v}, \\
    \left\|\mathcal{S}_{\mathcal{M}_{v_c}^{\rm enet}}\right\|_\infty &< \mathbf{1}.
	\end{aligned}
    \end{array}
    \right \}
\nonumber
\end{aligned}
\end{equation}
The last two conditions of $\cZ_v^{\rm enet}$ can be characterized by a set of inequalities as in Lemma \ref{lemma:cZ_v} with similar approach:
\begin{equation}
    \begin{aligned}
        &\quad\left\{\operatorname{sign}(\hat{\boldsymbol{\beta}}_{\mathcal{M}_v^{\rm enet}}(z)) = \mathcal{S}_{\mathcal{M}^{\rm enet}_v}\right\} \\ &= \left\{\mathcal{S}_{\mathcal{M}^{\rm enet}_v} \circ \hat{\boldsymbol{\beta}}_{\mathcal{M}_v^{\rm enet}}(z) > \mathbf{0}\right\}, \\
        &= \left\{\mathcal{S}_{\mathcal{M}^{\rm enet}_v} \circ (\tilde{X}_{u_{\cM_v^{\rm enet}}}^\top \tilde{X}_{u_{\cM_v^{\rm enet}}} + \gamma I)^{-1} (\tilde{X}_{u_{\cM_v^{\rm enet}}}^\top \tilde{\bm Y}_u(z) - \lambda\mathcal{S}_{\mathcal{M}^{\rm enet}_v}) > \mathbf{0}\right\}, \\
        &= \left\{\mathcal{S}_{\mathcal{M}^{\rm enet}_v} \circ \left(\tilde{X}_{u_{\cM_v^{\rm enet}}}\right)^* \tilde{\bm Y}_u(z) > \lambda \mathcal{S}_{\mathcal{M}^{\rm enet}_v} \circ \left(\left(\tilde{X}_{u_{\cM_v^{\rm enet}}}^\top \tilde{X}_{u_{\cM_v^{\rm enet}}} + \gamma I\right)^{-1} \mathcal{S}_{\cM_v^{\rm enet}}\right)\right\}, \\
        &= \left\{\mathcal{S}_{\mathcal{M}^{\rm enet}_v} \circ \left(\tilde{X}_{u_{\cM_v^{\rm enet}}}\right)^* \Omega_u {\bm Y}(z) > \lambda \mathcal{S}_{\mathcal{M}^{\rm enet}_v} \circ \left(\left(\tilde{X}_{u_{\cM_v^{\rm enet}}}^\top \tilde{X}_{u_{\cM_v^{\rm enet}}} + \gamma I\right)^{-1} \mathcal{S}_{\cM_v^{\rm enet}}\right)\right\}, \\
        &= \left\{\mathcal{S}_{\mathcal{M}^{\rm enet}_v} \circ \left(\tilde{X}_{u_{\cM_v^{\rm enet}}}\right)^* \Omega_u (\bm a + \bm b z) > \lambda \mathcal{S}_{\mathcal{M}^{\rm enet}_v} \circ \left(\left(\tilde{X}_{u_{\cM_v^{\rm enet}}}^\top \tilde{X}_{u_{\cM_v^{\rm enet}}} + \gamma I\right)^{-1} \mathcal{S}_{\cM_v^{\rm enet}}\right)\right\}, \\
        &= \left\{{\bm \psi}^{\rm enet}_0 z \leq {\bm \phi}^{\rm enet}_0\right\},\\
    \end{aligned}
    \nonumber
    \end{equation}
\begin{equation}
\begin{aligned}
    &\quad\left\{\left\|\mathcal{S}_{\mathcal{M}_{v_c}^{\rm enet}}\right\|_\infty < \mathbf{1}\right\} \\ &= \left\{-\mathbf{1} < \cS_{\cM_{v_c}}^{\rm enet} < \mathbf{1}\right\}, \\
    \qquad\qquad&= \left\{-\mathbf{1} < \tilde{X}_{u_{\cM_{v_c}^{\rm enet}}}^\top (\tilde{X}_{u_{\cM_v^{\rm enet}}})^{**} \mathcal{S}_{\mathcal{M}^{\rm enet}_v} + \frac{1}{\lambda} \tilde{X}_{u_{\cM_{v_c}^{\rm enet}}}^\top (I_{n_s+n_t} - \tilde{X}_{u_{\cM_v^{\rm enet}}}(\tilde{X}_{u_{\cM_v^{\rm enet}}})^*)\tilde{\bm Y}_u(z) < \mathbf{1}\right\}, \\
    &= \left\{\begin{aligned}
        \frac{1}{\lambda} \tilde{X}_{u_{\cM_{v_c}^{\rm enet}}}^\top \left(I_{n_s+n_t} - \tilde{X}_{u_{\cM_v^{\rm enet}}}\left(\tilde{X}_{u_{\cM_v^{\rm enet}}}\right)^*\right)\tilde{\bm Y}_u(z) &< \mathbf{1} - \tilde{X}_{u_{\cM_{v_c}^{\rm enet}}}^\top \left(\tilde{X}_{u_{\cM_v^{\rm enet}}}\right)^{**} \mathcal{S}_{\cM_v^{\rm enet}} \\
        - \frac{1}{\lambda} \tilde{X}_{u_{\cM_{v_c}^{\rm enet}}}^\top \left(I_{n_s+n_t} - \tilde{X}_{u_{\cM_v^{\rm enet}}}\left(\tilde{X}_{u_{\cM_v^{\rm enet}}}\right)^*\right)\tilde{\bm Y}_u(z) &< \mathbf{1} + \tilde{X}_{u_{\cM_{v_c}^{\rm enet}}}^\top \left(\tilde{X}_{u_{\cM_v^{\rm enet}}}\right)^{**} \mathcal{S}_{\cM_v^{\rm enet}}
    \end{aligned}\right\}, \\
    &= \left\{\begin{aligned}
        \frac{1}{\lambda} \tilde{X}_{u_{\cM_{v_c}^{\rm enet}}}^\top \left(I_{n_s+n_t} - \tilde{X}_{u_{\cM_v^{\rm enet}}}\left(\tilde{X}_{u_{\cM_v^{\rm enet}}}\right)^*\right) \Omega_u {\bm Y}(z) &< \mathbf{1} - \tilde{X}_{u_{\cM_{v_c}^{\rm enet}}}^\top \left(\tilde{X}_{u_{\cM_v^{\rm enet}}}\right)^{**} \mathcal{S}_{\cM_v^{\rm enet}} \\
        - \frac{1}{\lambda} \tilde{X}_{u_{\cM_{v_c}^{\rm enet}}}^\top \left(I_{n_s+n_t} - \tilde{X}_{u_{\cM_v^{\rm enet}}}\left(\tilde{X}_{u_{\cM_v^{\rm enet}}}\right)^*\right)\Omega_u {\bm Y}(z) &< \mathbf{1} + \tilde{X}_{u_{\cM_{v_c}^{\rm enet}}}^\top \left(\tilde{X}_{u_{\cM_v^{\rm enet}}}\right)^{**} \mathcal{S}_{\cM_v^{\rm enet}}
    \end{aligned}\right\}, \\
    &= \left\{\begin{aligned}
        \frac{1}{\lambda} \tilde{X}_{u_{\cM_{v_c}^{\rm enet}}}^\top \left(I_{n_s+n_t} - \tilde{X}_{u_{\cM_v^{\rm enet}}}\left(\tilde{X}_{u_{\cM_v^{\rm enet}}}\right)^*\right) \Omega_u (\bm a + \bm b z) &< \mathbf{1} - \tilde{X}_{u_{\cM_{v_c}^{\rm enet}}}^\top \left(\tilde{X}_{u_{\cM_v^{\rm enet}}}\right)^{**} \mathcal{S}_{\cM_v^{\rm enet}} \\
        - \frac{1}{\lambda} \tilde{X}_{u_{\cM_{v_c}^{\rm enet}}}^\top \left(I_{n_s+n_t} - \tilde{X}_{u_{\cM_v^{\rm enet}}}\left(\tilde{X}_{u_{\cM_v^{\rm enet}}}\right)^*\right)\Omega_u (\bm a + \bm b z) &< \mathbf{1} + \tilde{X}_{u_{\cM_{v_c}^{\rm enet}}}^\top \left(\tilde{X}_{u_{\cM_v^{\rm enet}}}\right)^{**} \mathcal{S}_{\cM_v^{\rm enet}}
    \end{aligned}\right\}, \\
    &= \left\{ 
    \begin{pmatrix}
        {\bm \psi}_{10}^{\rm enet} \\
        {\bm \psi}_{11}^{\rm enet}
    \end{pmatrix} z \leq 
    \begin{pmatrix}
        {\bm \phi}_{10}^{\rm enet} \\
        {\bm \phi}_{11}^{\rm enet}
    \end{pmatrix}\right\}, \\
    &= \left\{ {\bm \psi}_1^{\rm enet} z \leq {\bm \phi}_1^{\rm enet} \right\}.
    \end{aligned}
\nonumber
\end{equation}
Finally, the set $\mathcal{Z}^{\rm enet}_v$ can also be identified by solving a set of linear inequalities w.r.t $z$:
\begin{equation}
\begin{aligned}
    \mathcal{Z}^{\rm enet}_v &= \left\{z \in \mathbb{R} \mid {\bm \psi}^{\rm enet} z \leq {\bm \phi}^{\rm enet} \right\},
\end{aligned}
\nonumber
\end{equation}
where ${\bm \psi}^{\rm enet} = \left({\bm \psi}_0^{\rm enet} \quad {\bm \psi}_1^{\rm enet}\right)^\top$, ${\bm \phi}^{\rm enet} = \left({\bm \phi}_0^{\rm enet} \quad {\bm \phi}_1^{\rm enet}\right)^\top$.

\subsection{Additional Experiments} \label{app:additional_experiment}

\textbf{Real-World Datasets.}
As noted in \S\ref{sec:experment}, we also we conducted a comparison of different methods on two datasets: the CO2 Emissions Canada dataset and the Walmart dataset, both available on Kaggle. We also present the percentage of \textit{p}-value less than or equal $\alpha$ on each dataset, which can be used for providing insights into the statistical power of each method.
\begin{figure}[!t]
    \centering
    \begin{subfigure}[b]{0.85\linewidth}  
        \includegraphics[width=\linewidth]{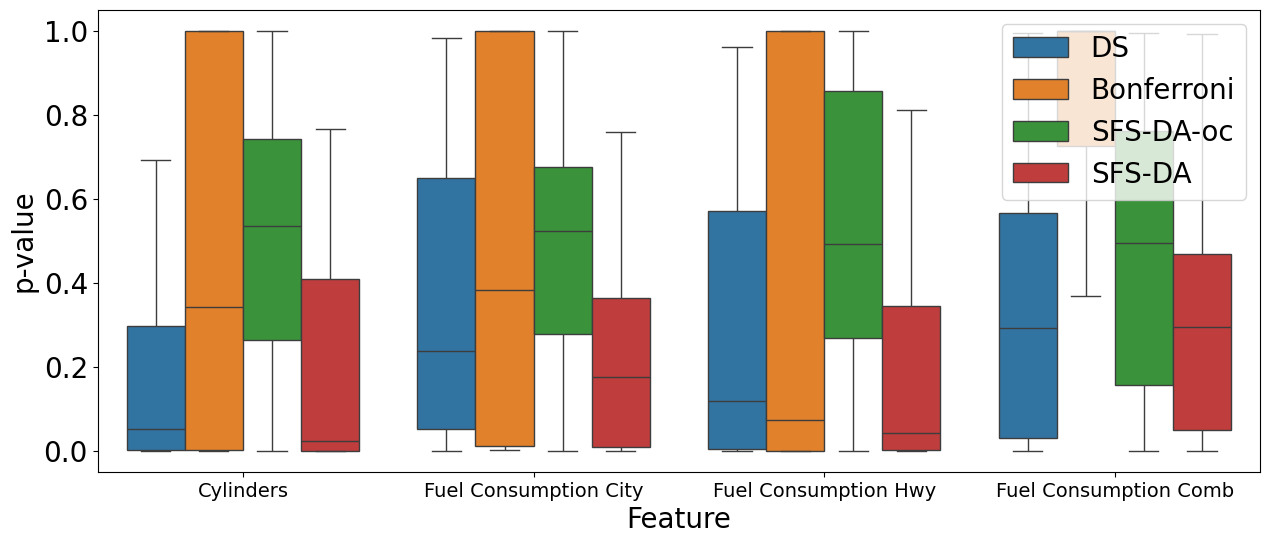}  
        \caption{CO2 Emissions Canada dataset. The source domain comprises ``vehicles using gasoline fuel'', while the target domain includes ``vehicles that use other types of fuel''.}
    \end{subfigure}

    \vspace{20pt}
    
    \begin{subfigure}[b]{0.85\linewidth}
        \includegraphics[width=\linewidth]{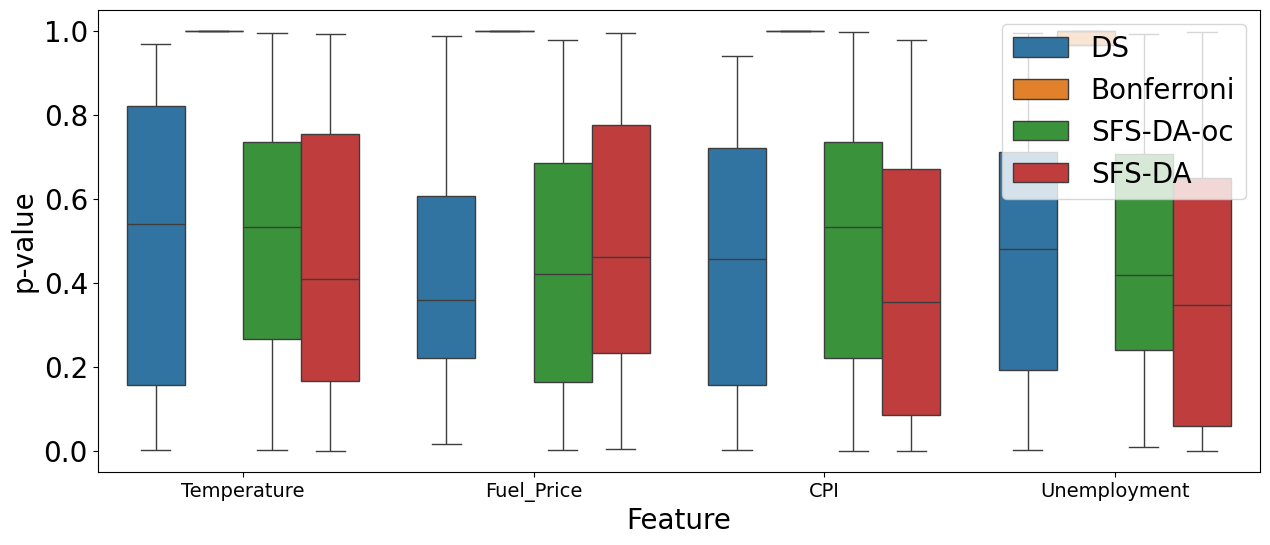}
        
        \caption{Walmart dataset. The source domain is ``people who go shopping at Walmart on regular days'', while the target domain is ``people who go shopping at Walmart on holidays''.}
    \end{subfigure}
    
    \caption{Distributions of \textit{p}-values for each feature in the CO2 Emissions Canada and Walmart datasets. While the median \textit{p}-value for {\tt DS} is smaller than that of the proposed {\tt SFS-DA} in one case (Fuel Price in the Walmart dataset), in all other cases, the median \textit{p}-value for {\tt SFS-DA} is smaller than those of the other methods, indicating that {\tt SFS-DA} demonstrates the superior statistical power.}
    \label{fig:boxplot co2 and walmart}
\end{figure}

\begin{figure}[!t]
    \centering
    \begin{subfigure}[b]{0.48\linewidth}  
    \centering
        \includegraphics[width=.9\linewidth]{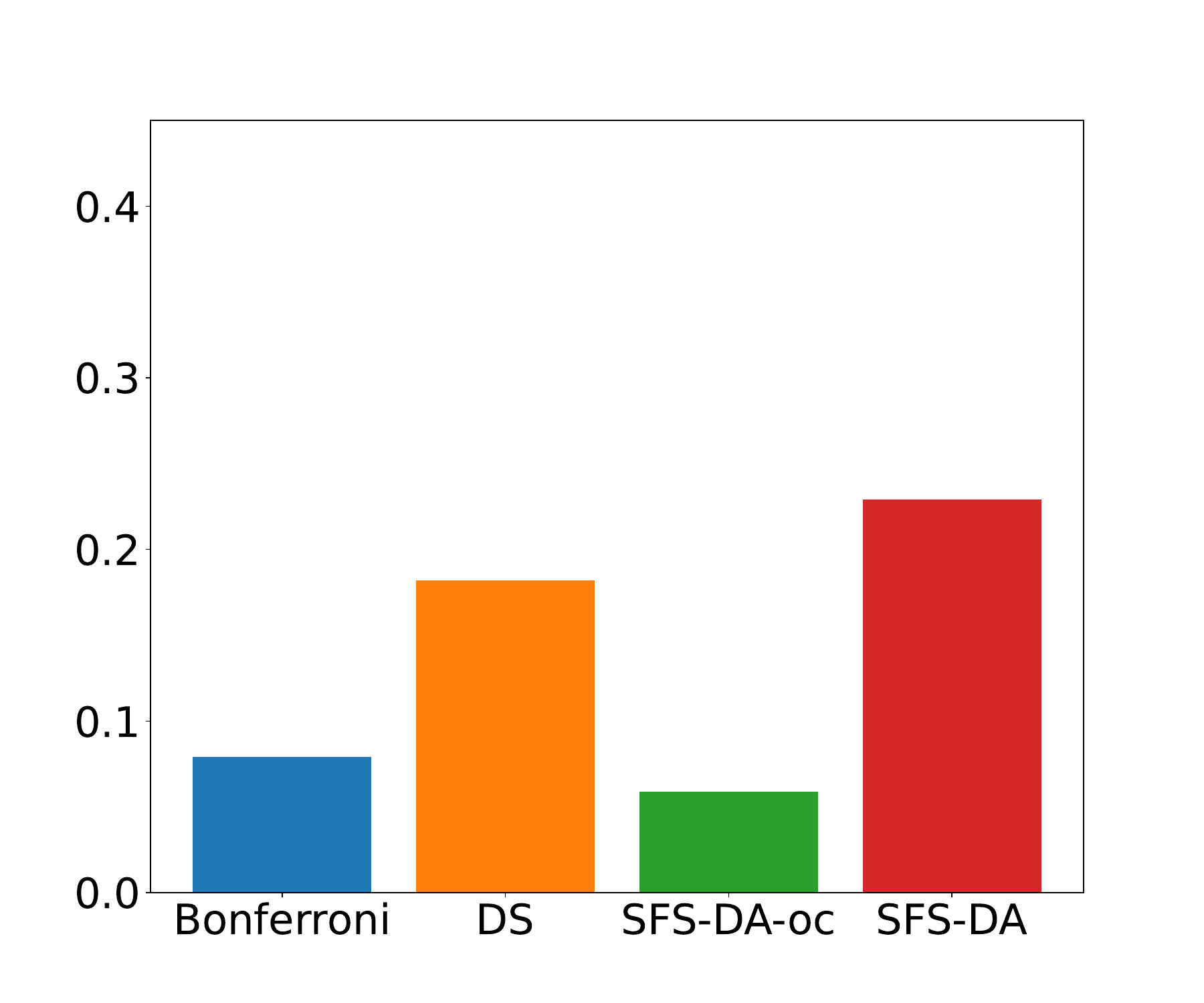}  
        \caption{Diabetes}
        \label{fig:percentage diabetes}
    \end{subfigure}
    \hspace{0.02\linewidth}  
    \begin{subfigure}[b]{0.48\linewidth}
    \centering
        \includegraphics[width=.9\linewidth]{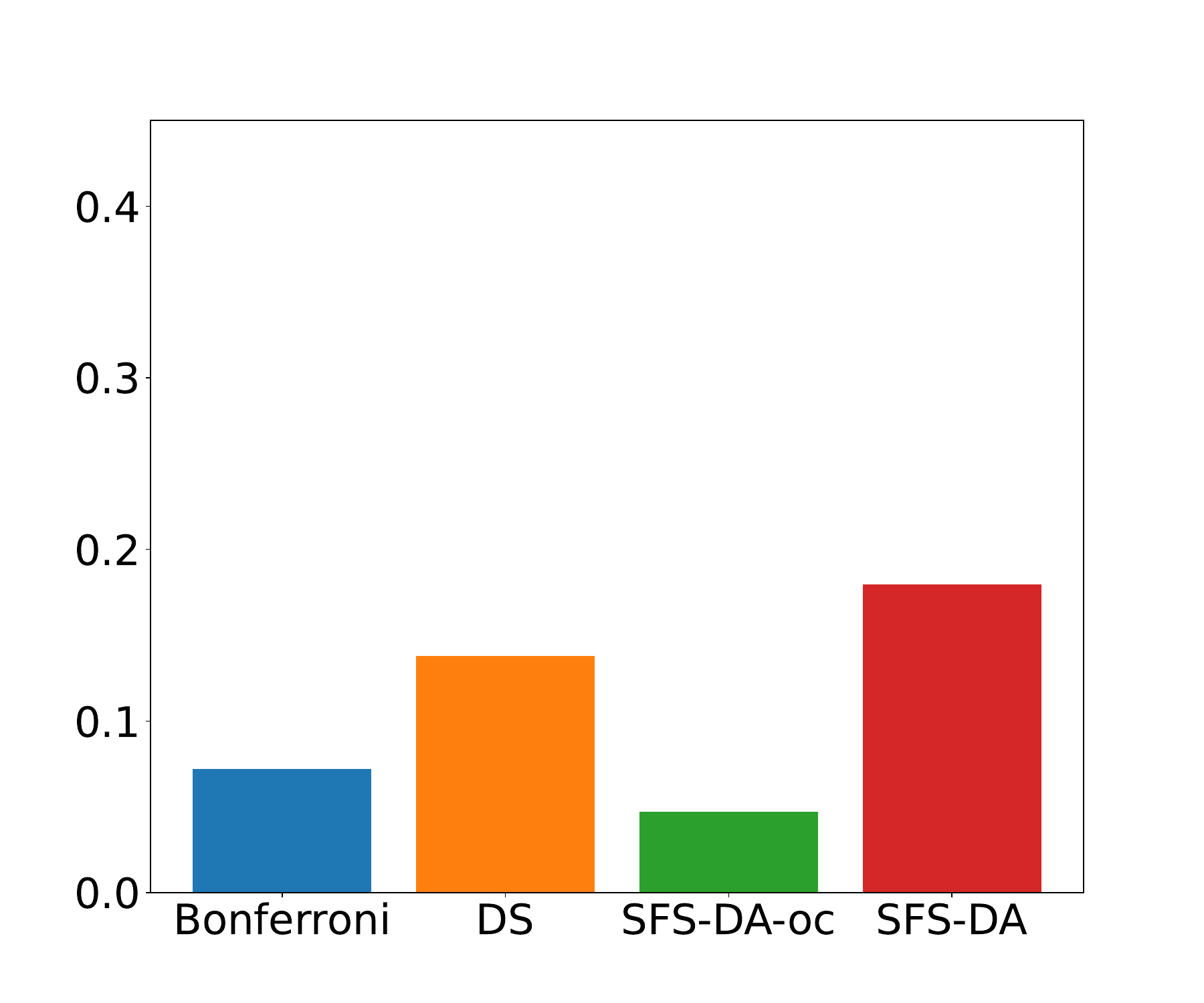}
        \caption{Heart Failure}
        \label{fig:percentage heart failure}
    \end{subfigure}

\begin{subfigure}[b]{0.48\linewidth}  
\centering
        \includegraphics[width=.9\linewidth]{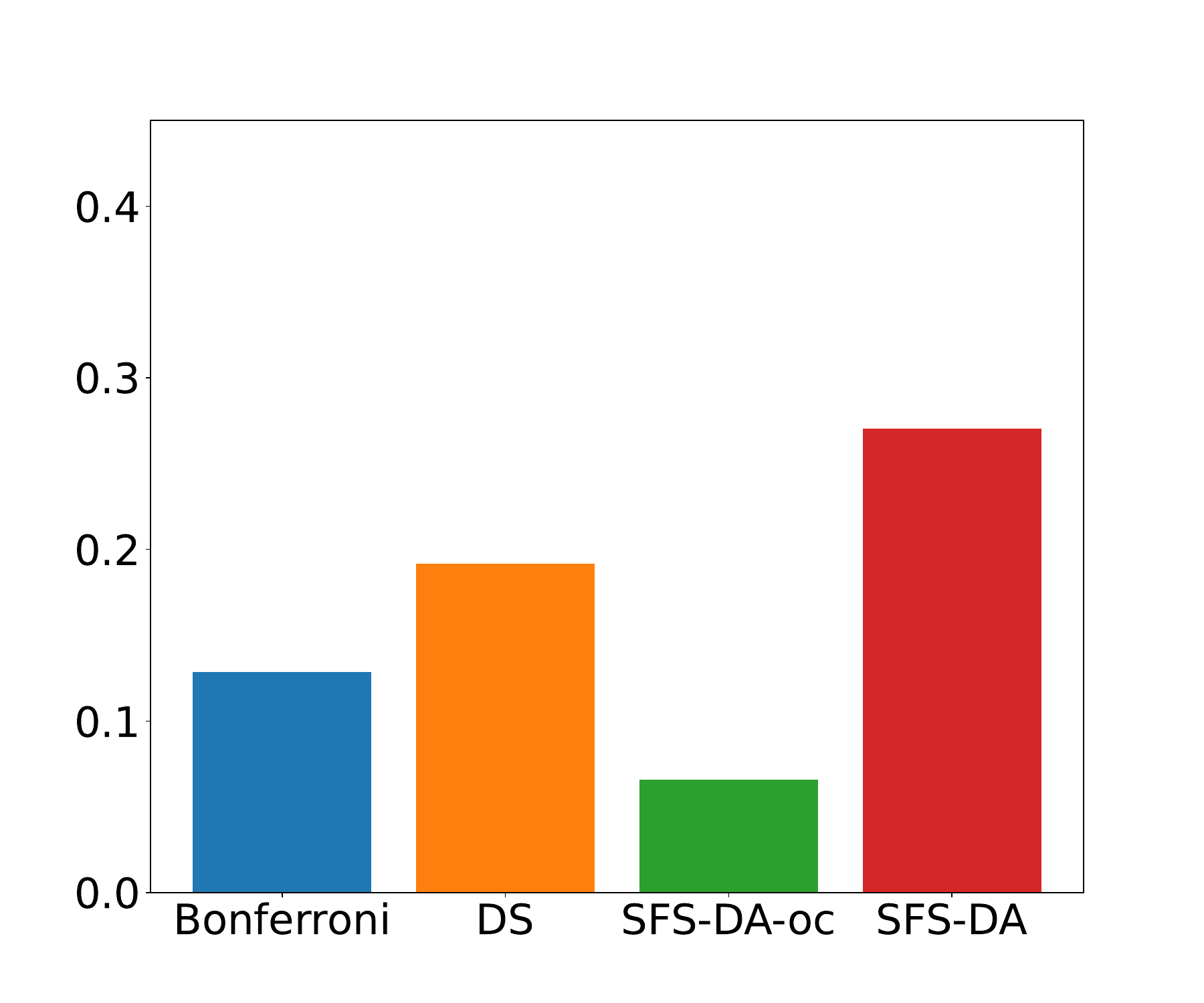}  
        \caption{Seoul Bike}
        \label{fig:percentage seoul bike}
    \end{subfigure}
    \hspace{0.02\linewidth}  
    \begin{subfigure}[b]{0.48\linewidth}
    \centering
        \includegraphics[width=.9\linewidth]{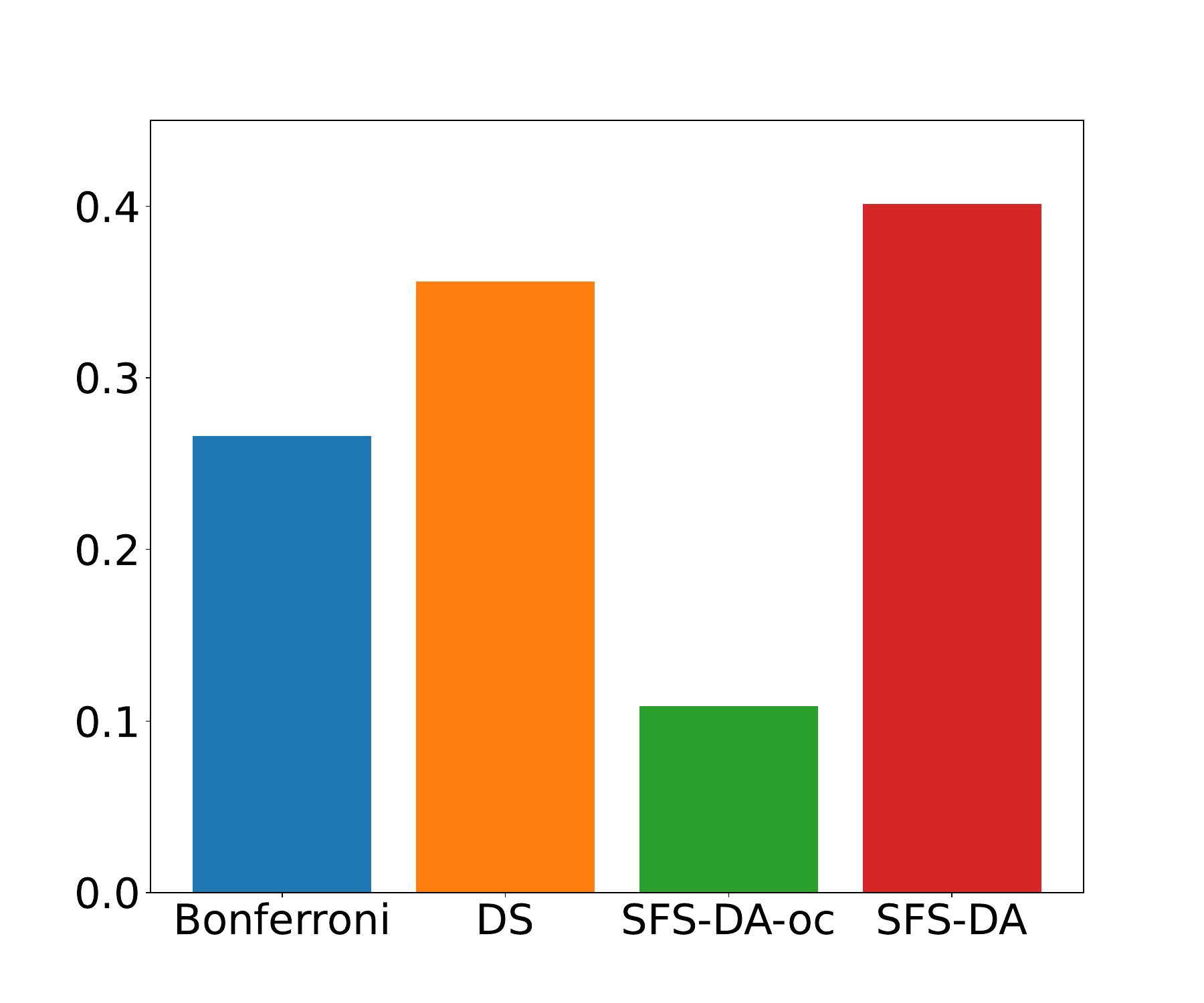}
        \caption{CO2 Emissions Canada}
        \label{fig:percentage co2}
    \end{subfigure}

    \begin{subfigure}[b]{0.47\linewidth}
        \includegraphics[width=.9\linewidth]{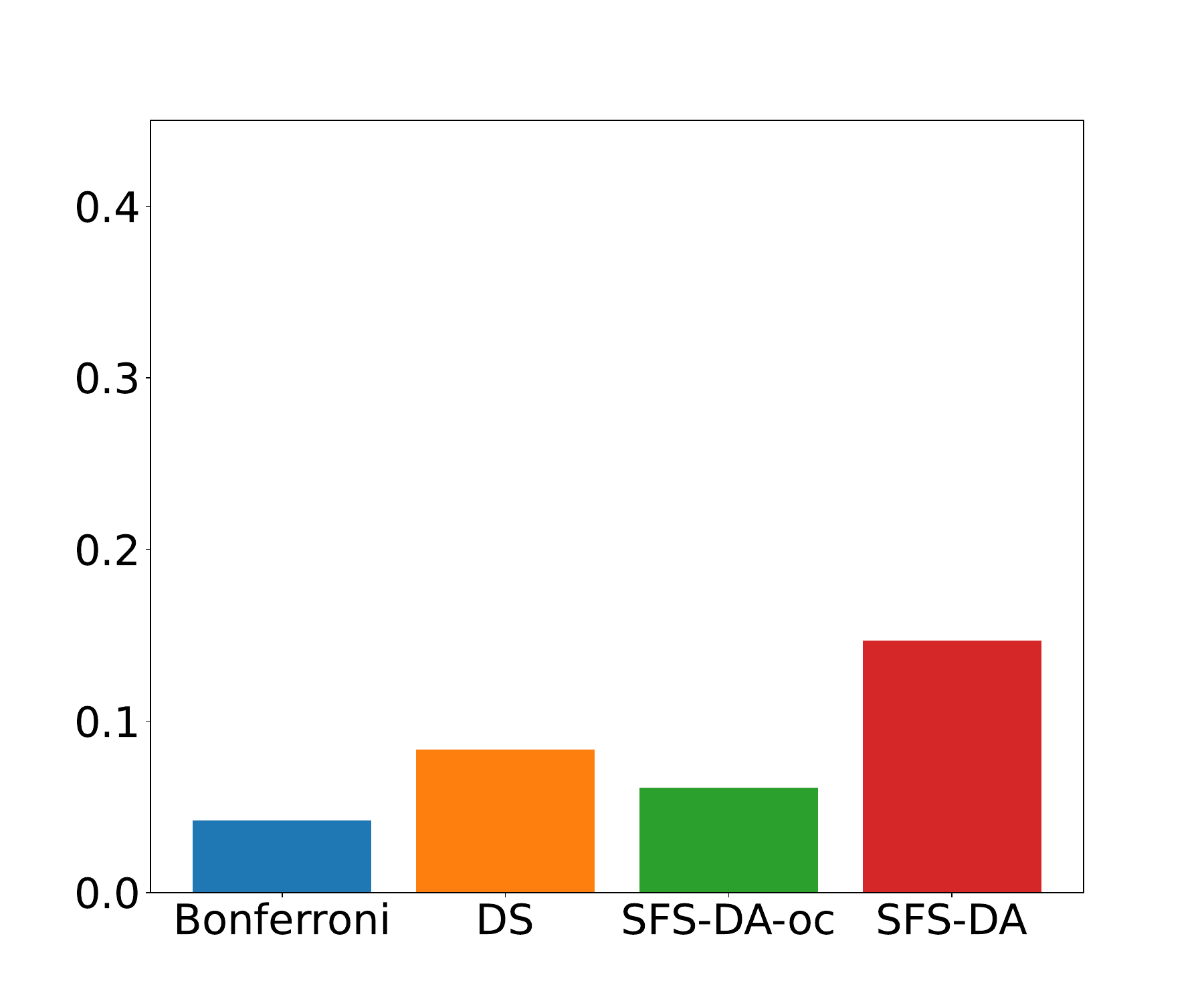}
        \caption{Walmart}
        \label{fig:percentage walmart}
    \end{subfigure}
    
    \caption{$\mathbb{P}(p\text{-value}\leq\alpha)$ on each dataset, where $\alpha = 0.05$. In this setting, we randomly select one feature from the set of selected features to conduct inference. In all cases, the percentage of significant \textit{p}-values from the proposed method is higher than that of the competing methods. This suggests that our proposed {\tt SFS-DA} method exhibits higher statistical power compared to the alternatives.}
    \label{fig:percentage}
\end{figure}